\documentclass[twoside,11pt]{article}

\usepackage[preprint]{jmlr2e}

\usepackage{subfig}
\usepackage{wrapfig}

\firstpageno{1}

\usepackage{amsmath}
\usepackage{amssymb}
\usepackage{graphicx}
\usepackage{amsmath}
\usepackage{xcolor}
\usepackage{algorithm}
\usepackage{algorithmic}
\usepackage{bm}
\usepackage[normalem]{ulem}

\usepackage{enumitem}

\hypersetup{colorlinks=true,citecolor=blue,linkcolor=blue}

\DeclareMathOperator*{\nth}{^{\text{th}}}

\newcommand{\Deltamin}{\Delta_{\min}}

\newcommand{\J}{\mathcal{J}}

\newcommand{\lev}{\ell} 

\newenvironment{proof-sketch}{\par\noindent{\bf Proof sketch\ }}{\hfill\BlackBox\\[2mm]}

\begin{document}

\title{Near-Optimal Algorithms for Differentially Private Online Learning in a Stochastic Environment\thanks{
The previous version of this arXiv preprint, \href{https://arxiv.org/abs/2102.07929v2}{2102.07929v2}, also contained a theoretical analysis of the algorithm Hybrid-UCB. Although that analysis is correct to our knowledge, it has been excluded from the current version. New to the current version are theoretical results for the algorithm Lazy-DP-TS which previously appeared in our UAI 2022 paper \citep{hu2022near}.
}}

\author{\name Bingshan Hu \email bingsha1@cs.ubc.ca \\
      \addr 
      University of British Columbia,
      Vancouver, Canada 
      \AND
      \name Zhiming Huang \email zhiminghuang@uvic.ca \\
      \addr  University of Victoria, Victoria, Canada
      \AND
      \name Nishant A. Mehta \email nmehta@uvic.ca \\
      \addr  University of Victoria, Victoria, Canada
      \AND
      \name Nidhi Hegde \email nidhi.hegde@ualberta.ca \\
      \addr 
      University of Alberta, Alberta Machine Intelligence Institute,
      Edmonton, Canada 
}

\editor{}

\maketitle

\begin{abstract}

In this paper, we study differentially private online learning problems in a stochastic environment under both bandit and full information feedback. For  differentially private stochastic bandits, we propose both   UCB and Thompson Sampling-based algorithms that are anytime and achieve the optimal $O \left(\sum_{j: \Delta_j>0} \frac{\ln(T)}{\min \left\{\Delta_j, \epsilon \right\}} \right)$ instance-dependent regret bound, where $T$ is the finite learning horizon, $\Delta_j$ denotes the suboptimality gap between the optimal arm and  a suboptimal arm $j$, and $\epsilon$ is the required privacy parameter. 
For the differentially private full information setting with stochastic rewards, we show an  $\Omega \left(\frac{\ln(K)}{\min \left\{\Delta_{\min}, \epsilon \right\}} \right)$ instance-dependent regret lower bound and an $\Omega\left(\sqrt{T\ln(K)} + \frac{\ln(K)}{\epsilon}\right)$ minimax lower bound, where $K$ is the total number of actions and $\Delta_{\min}$ denotes the minimum suboptimality gap among all the suboptimal actions. For the same differentially private full information setting, we also present an $\epsilon$-differentially private algorithm whose instance-dependent regret and worst-case regret match our respective lower bounds up to an extra $\log(T)$ factor.
\end{abstract}

\


\section{Introduction} \label{sec: intro}

This work considers
stochastic online learning settings where, in each of a sequence of rounds, a learning agent selects one of $K$ actions, collects an independent stochastic reward associated with that action, and then receives some feedback. 
We assume the stochastic reward sequence for any action $j \in [K]$ is i.i.d.~according to a Bernoulli distribution with mean $\mu_j$. 
There are two commonly-studied feedback models: the \emph{bandit} setting --- where only the reward of the played action is revealed ---  and the \emph{full information} setting --- where the rewards of all actions are revealed. As is customary, we will refer to actions as \emph{arms} in the bandit setting, while in general we use the term \emph{actions}.
The learning agent wishes to maximize its cumulative reward over $T$ rounds. 
Naturally, 
the learning agent 
can and should
use the revealed rewards from previous rounds to decide what actions to play in future rounds. 
However, in many settings, rewards themselves may encode private information that should be protected. 
For instance, consider a system for search advertising
where the objective is to display relevant ads for web queries. In such a setting the system would display a few advertisements to the user.  When the user clicks on an ad, a reward is collected by the system, which in accumulation would allow the system and any external observers to learn user preferences. Rewards that represent user preferences are private information and may further allow inference on the user's other private characteristics. 
Motivated by the need to protect privacy, in this work we study differentially private stochastic online learning under bandit feedback and full information feedback. 
The framework of differential privacy is widely accepted and is based on the notion of plausible deniability: from a distributional point of view, an adversary should learn almost the same thing if one element in the dataset is changed or missing \citep{dwork2014algorithmic}. In the context of online learning, a dataset is the stream of reward vectors fed into the algorithm, and a change to one element refers to a change to one reward vector in the stream. 

We now introduce the bandit and full information settings in turn, highlighting the open questions we will address along the way.

\paragraph{(I) Bandit setting.}
The stochastic multi-armed bandit problem admits myriad applications and has attracted significant theoretical and practical attention over the past few decades. 
We adopt as our performance measure 
the \emph{pseudo-regret} (hereafter, simply ``regret''), which is the difference between the expected cumulative reward of the best action in expectation and the expected cumulative reward of the algorithm. 
A key challenge in designing low-regret algorithms for this problem is the need to balance exploitation and exploration. Addressing this challenge has led to numerous, rather different algorithms, including Upper Confidence Bound (UCB)-based \citep{auer2002finite,audibert2009exploration,garivier2011kl,kaufmann2012bayesian,lattimore2018refining}, Thompson Sampling-based \citep{agrawal2017near,kaufmann2012thompson,bian2022maillard,jin2021mots,jin2022finite,jin2023thompson,hu2020problem,hu2023optimistic,Opt2023}, and elimination-based \citep{even2002pac,auer2010ucb} algorithms. 
Much has been settled in this setting: we have algorithms that are asymptotically optimal, algorithms that are minimax optimal in rate, and algorithms that satisfy a certain notion of instance-dependent optimality known as the sub-UCB property (further explained below); for at least one subclass of problems, a single algorithm achieves all three properties simultaneously \citep{lattimore2018refining}.
Moreover, algorithms have been developed that admit good theoretical guarantees while also having strong practical performance; Thompson Sampling \citep{agrawal2017near} is a notable example.

Before venturing to the private setting, it will be useful to introduce the sub-UCB property. 
Following \cite{lattimore2018refining}, we say an algorithm is \emph{sub-UCB} if its regret is
\begin{align}
O \left( \sum_{j \in [K] \colon \Delta_j > 0} \frac{\log(T)}{\Delta_j} \right) \quad, \label{eqn:sub-ucb}
\end{align}
where $\Delta_j = \max_{i \in [K]} \mu_i - \mu_j$ denotes the mean reward gap for arm $j \in [K]$ and indicates the single round performance loss when the learning agent pulls a sub-optimal arm $j$ instead of the optimal arm. Instance-dependent optimality is nuanced: when considering regret bounds depending on the gaps $(\Delta_1, \dotsc, \Delta_K)$ but not directly depending on the means $(\mu_1, \dotsc, \mu_K)$, the sub-UCB property is a suitable notion of instance-dependent optimality due to known instance-dependent lower bounds \cite[Chapter 16]{lattimore2020bandit}.\footnote{Improved instance-dependent bounds are sometimes possible when the bounds can depend on the means.} We therefore adopt the sub-UCB property as our notion of instance-dependent optimality.

The state of affairs is much less settled when one requires differential privacy.
\cite{shariff2018differentially} proved that
for any algorithm that 
satisfies a very mild notion of consistency\footnote{Namely, always enjoying $O(T^{3/4})$ regret, although the particular exponent of $3/4$ is not important.},
the requirement of $\epsilon$-differential privacy (see Section~\ref{sec:dp-def} for a formal definition) comes at the price of 
$\Omega \left( \frac{K \log(T)}{\epsilon} \right)$ regret on any problem instance.
Together with known sub-UCB lower bounds, the optimal rate for instance-dependent regret (in terms of the gaps) is
\begin{align}
O \left( \sum_{j \in [K] \colon \Delta_j > 0}
             \frac{ \log (T)}{\min \left\{\Delta_j, \epsilon\right\}}
  \right) \quad; 
\label{eqn:private-sub-ucb}
\end{align}
for brevity, we say an algorithm achieving the above rate is \emph{private sub-UCB}.
Shortly after \cite{shariff2018differentially} proved their lower bound, \citet{sajed2019optimal} designed Differentially Private Successive Elimination (DP-SE), an elimination-style algorithm whose non-private regret term is sub-UCB and whose private regret term matches the lower bound of \cite{shariff2018differentially}. 
Being an elimination-style algorithm, DP-SE requires knowledge of the time horizon $T$ and therefore is not an anytime algorithm. Although DP-SE could be made anytime by using the doubling trick, this comes at the cost of practical performance. On the other hand, in the non-private setting, UCB and Thompson Sampling are sub-UCB and ``natively'' anytime; moreover, they both typically well-outperform elimination-style algorithms, with Thompson Sampling exhibiting particularly good practical performance. Therefore, it is natural to ask if one can develop differentially private UCB-style and Thompson Sampling-style algorithms with regret bounds of the same order as DP-SE while also being anytime. This is indeed possible, and the resulting algorithms constitute two of our core contributions.

\paragraph{(II) Full information setting.}
Upgrading to full information feedback and allowing the learning agent to hedge by playing a distribution over actions in each round, we obtain an interesting special case of online stochastic convex optimization with close connections to (batch) statistical learning. Technically, we obtain a stochastic variant of 
decision-theoretic online learning (DTOL), a simple specialization of the game of prediction with expert advice. 
Unlike the case of bandit feedback, exploration is not needed since the rewards of all actions are revealed regardless of which action was taken.
Consequently, 
it is much less challenging to develop an optimal algorithm. 
Let $\Deltamin$ be the minimum gap between the mean reward of any suboptimal arm and the mean reward $\max_{j \in [K]} \mu_j$ of the optimal arm.
Follow-the-Leader (FTL) \citep{vanerven2014follow}, a simple learning algorithm that plays the action with the highest empirical mean, achieves the optimal $O \left(\frac{\log(K)}{\Deltamin} \right)$ instance-dependent regret bound \citep{kotlowski2019private} and is actually optimal in a much stronger sense \citep{kotlowski2018minimaxity}.
It may therefore come as a surprise that, compared to the bandit setting, 
there are even larger gaps in our understanding of the stochastic full information setting under the constraint of differential privacy. 
First, while an instance-independent regret upper bound was recently shown by \cite{asi2022private}, it is unclear whether the rate they obtain is optimal. 
Next, to our knowledge, 
no prior work has given instance-dependent regret upper bounds, and lower bounds also are yet to be shown. We make progress on both these fronts. Moreover, via a standard argument, our instance-dependent upper bound can be used to recover the instance-independent upper bound of  \cite{asi2022private}.

\subsection{Preview of Results}
\label{sec:preview}

We begin with our contributions in the bandit setting. Our first contribution is Anytime-Lazy-UCB (Algorithm~\ref{Optimal DP-UCB}), the first anytime algorithm for differentially private stochastic bandits that satisfies the private sub-UCB property shown in \eqref{eqn:private-sub-ucb}.
Note that while DP-SE is also private sub-UCB, it is not anytime. 
Our second contribution is Lazy-DP-TS (Algorithm~\ref{Optimal DP-TS}), the first Thompson Sampling-based algorithm for differentially private stochastic bandits that is private sub-UCB; this algorithm is also anytime.

Anytime-Lazy-UCB and Lazy-DP-TS share the same type of arm-specific schedule for updating the empirical mean of each arm. 
This schedule is a critical design element for the algorithms to be private sub-UCB. The schedule works as follows. The outer part of the algorithm (a suitably modified UCB or Thompson Sampling rule) decides, given a stored private empirical mean for each arm, an arm to pull. 
Control then passes to the inner part of the algorithm, which maintains a separate epoch schedule for each arm and is responsible for updating the private empirical mean of the arm. For a given arm, the epoch length (i.e., the number of pulls) doubles in each successive epoch. 
The inner part only updates the private empirical mean of the corresponding arm at the end of the arm's current epoch; moreover, the private empirical mean is computed using only the observations from the pulls within this current epoch.
Therefore, for any arm, (a) the updating of the empirical mean is lazy (it only happens at the end of 
the arm's  epoch); (b) it is forgetful, as the previous observations (from the arm's previous epochs) are forgotten; and (c) all non-pulled arms do not progress in their respective schedules. This way of separately maintaining a 
``doubling'' schedule 
for each arm
can be viewed as a generalization of the geometric schedule adopted by DP-SE. 
Note that the schedules of all arms are synchronized in DP-SE, 
whereas this is not the case in our algorithm design template.

As would be expected, our design of Anytime-Lazy-UCB follows the principle of optimism in the face of uncertainty (OFU). Beyond the use of arm-specific epoch schedules, Anytime-Lazy-UCB adds Laplace noise to the empirical mean of each arm to protect privacy. To ``cancel'' the impact of the added Laplace noise, Anytime-Lazy-UCB adds an extra exploration bonus term to  the non-private UCB index to ensure that the private UCB index is indeed an upper confidence bound for  arm's true mean. Unexpectedly, our design of Lazy-DP-TS  is also inspired by the principle of OFU.
Different from Anytime-Lazy-UCB where an extra exploration bonus term is  added directly to the non-private UCB index, in Lazy-DP-TS, we add a bonus term to the mean of the (differentially private) posterior distribution. Intuitively, we shift each arm's (differentially private) posterior distribution to the right as compared to the (non-private) posterior distribution, thereby ensuring that a posterior sample drawn from the shifted posterior distribution is more likely to be optimistic as compared to drawing a sample from the non-private posterior distribution.

In the full information setting, we present RNM-FTNL, a differentially private algorithm that achieves an instance-dependent regret
\begin{align}
O \left( 
    \frac{ \log (K)}{ \Deltamin} 
    + \frac{\log(K) \log \left( \Delta_{\max}/\max\left\{ \Deltamin, \epsilon, 1/T \right\} \right)}
           {\epsilon}
  \right) \quad,
  \label{eqn:full-info-inst-dep-upper}
\end{align}
where $\Deltamin = \min_{j \in [K]: \Delta_j >0} \Delta_j$ and $\Delta_{\max} = \max_{j \in [K]} \Delta_j$ are the minimum and maximum mean reward gap respectively. To our knowledge, this is the first algorithm in the differentially private full information setting for which an instance-dependent regret bound has been shown. As a simple consequence of our instance-dependent regret bound, we also show an instance-independent regret bound of order
\begin{align}
O \left( \sqrt{T \log (K)} + \frac{\log(K) \log(T)}{\epsilon} \right) \quad, 
 \end{align}   
matching a rate recently achieved by \cite{asi2022private}.

Similar to Anytime-Lazy-UCB and Lazy-DP-TS, the design of RNM-FTNL crucially relies on geometrically increasing epochs,
forgetting past observations, and lazy updating. The algorithm itself is simple: because all actions' rewards are observed in each round, we use a global epoch schedule, in each epoch playing for all rounds the ``noisy leader'' based on the statistics from only the previous epoch. Here, the noisy leader is simply the leader based on empirical means with added Laplace noise of a suitable scale. The analysis is more interesting. First, since the full information setting allows the algorithm's output to depend on complete reward vectors, it is tempting to increase the scale of each action's Laplace noise to order $K / \epsilon$; but this would spoil the regret bound. By adopting a Report-Noisy-Max (hence ``RNM'') view of our algorithm, we manage to keep the noise at a scale of $1/\epsilon$. Second, to obtain our bounds, our analysis relies on a ``grouping trick'' that partitions the actions into disjoint groups, with successive groups having geometrically larger gaps. Due to the epoch structure of our algorithm and its forgetfulness, it may seem natural to bound the regret separately for each epoch; this is precisely what \cite{asi2022private} do. By instead using the grouping trick (together with a more or less standard FTL analysis), we bound the regret separately for each group while tuning a group-dependent horizon globally (i.e., for all epochs). As we explain in Section~\ref{sec:full-info}, our style of analysis allows us to obtain better regret bounds in certain important problem instances.

In addition, for the same setting, 
we present the first 
instance-dependent\footnote{Similar to the bandit setting, here we mean dependence on the gaps but not necessarily the means.} regret lower bound, of order
\begin{align}
    \Omega \left(\frac{\log (K)}{\Deltamin} + \frac{\log (K)}{\epsilon} \right) \quad. 
    \label{eqn:full-info-inst-dep-lower}
\end{align}
 This lower bound also implies a minimax regret lower bound of order
\begin{align}
  \Omega \left(\sqrt{T \log(K)} + \frac{\log (K)}{\epsilon} \right)\quad.
\end{align}
The form of the upper bound \eqref{eqn:full-info-inst-dep-upper} may seem unsatisfying, especially since it does not match the lower bound \eqref{eqn:full-info-inst-dep-lower}. Intriguingly and as we further discuss in Section~\ref{sec:conclusion}, in the special case where all suboptimal arms have the same gap of $\Deltamin$, a refined argument gives an upper bound that matches the rate \eqref{eqn:full-info-inst-dep-lower} of the lower bound. That is, for the instance which intuitively is the hardest, our instance-dependent upper and lower bounds match. Yet, these bounds clearly do not match in general.

\section{Problem Setting}

We now formally present the stochastic online learning setting under both bandit and full information feedback, followed by formally introducing differential privacy.

\subsection{Bandit and Full Information Settings} \label{non-private online learning setting}

In a stochastic multi-armed bandit problem, 
there is 
a set $[K]$ of arms and a stochastic environment. At the beginning of each round $t = 1,2, \dotsc$, the environment generates a  reward vector $X(t) := \left(  X_1(t),X_2(t), \dotsc, X_K(t) \right)$, where each $X_{j}(t) \in \{0, 1\}$ is i.i.d.~over time from a fixed but unknown Bernoulli distribution with parameter $\mu_j \in (0,1)$.\footnote{Our learning algorithms can be generalized to any bounded reward setting by using the trick shown in \citet{agrawal2012analysis}.}
Simultaneously, the learning agent pulls an arm  $J_t \in [K]$. The learning agent then observes and obtains $X_{J_t}(t)$, the reward associated with the pulled arm. 
The full information setting is identical to the bandit setting, except that upon playing an action $J_t \in [K]$, the learning agent observes the complete reward vector $X(t)$.\footnote{The full information setting we consider is a specialization of DTOL, which (a) typically is studied under losses and (b) instead of playing a single action $J_t \in [K]$, the learning agent plays a probability distribution over all actions in each round. To unify the presentation with bandit setting, we use reward vectors and the learning agent plays a single action.} 

The goal of the learning agent is to pull arms (or take actions) sequentially to accumulate as much reward as possible over a finite number of $T$ rounds. Without loss of generality, we assume that the first arm is the unique optimal arm, i.e., $\mu_1 > \mu_j$ for all $j \ne 1$. Regret $\mathcal{R}(T)$ is used to measure the performance of the learning agent's decisions for pulling arms (or taking actions),  expressed as
\begin{equation}
\begin{array}{lll}
\mathcal{R}(T) &:= & 
T \cdot \mu_1 - \mathbb{E} \left[\sum\limits_{t=1}^{T} \mu_{J_t}  \right]  \quad,
\end{array}
\label{regret def}
\end{equation}
where the expectation is taken over the learning agent's sequential decisions $\left(J_1, J_2, \dotsc, J_T\right)$.

\subsection{Differentially Private Online Learning} \label{sec:dp-def}

More formally, let $X_{1:t}$ be the sequence of reward vectors up to round $t$ and let $X'_{1:t}$ be a neighbouring sequence which differs in at most one reward vector, 
say, at some round $s \leq t$. 
The output of an online learning algorithm is the sequence of decisions  $\left(J_1, J_2, \dotsc, J_t \right)$ made up to round $t$.  In the context of online learning, differential privacy is defined as follows \citep{agarwal2017price,sajed2019optimal}.

\begin{definition}[Differential privacy in  online learning]
A randomized online learning algorithm  $\mathcal{M}$ is $\epsilon$-DP if for any two reward vector sequences $X_{1:T}$ and $X'_{1:T}$ differing in at most one vector, for any decision set $\mathcal{D}_{1:t} \subseteq [K]^t$, we have  $\mathbb{P} \left\{\mathcal{M}({X}_{1:t}) \in \mathcal{D}_{1:t}  \right\} \le e^{\epsilon} \cdot \mathbb{P} \left\{\mathcal{M}({X}_{1:t}') \in \mathcal{D}_{1:t} \right\} $ holds for all $t \le T$ simultaneously.
\end{definition}
The above definition  follows the standard notion introduced in \citet{dwork2014algorithmic}. It can be interpreted in terms of 
the max-divergence $D_{\infty}(Q, Q') := \mathop{\max}_{y \in \text{supp}(Q')} \ln \left(\frac{P (Q=y)}{P(Q'=y)} \right)$  between two probability distributions $Q$ and $Q'$, where $Q$ denotes the output distribution (the distribution of the sequentially pulled arms) when working over the true reward sequences ${X}$ and $Q'$ denotes the output distribution when working over  ${X}'$.
  An $\epsilon$-differentially private algorithm $\mathcal{M} (\cdot)$ ensures that the max-divergence between $Q$ and $Q'$ is at most $\epsilon$, i.e., $\ln \left( \frac{P \left(Q=y\right)}{P \left(Q'=y \right)} \right) \le \epsilon$ for all possible outputs $y$. 
The value of  $\ln \left( \frac{P \left(Q=y\right)}{P \left(Q'=y \right)} \right)$ quantifies the privacy loss incurred when an adversary witnesses an outcome $y$.

\section{Background and Related Work} \label{sec:literature}

In this section, we describe some of the history of non-private and private learning in the bandit and full information settings.

\subsection{Learning Under Bandit Feedback}

\paragraph{Non-private learning.}

It will be useful to review the algorithm design principles behind UCB-based, Thompson Sampling-based, and elimination-based algorithms. 
Essentially, all these algorithms follow the principle of OFU, and rely on tracking the past empirical means and the number of pulls of each arm to make future decisions. The key difference between UCB-based and Thompson Sampling-based algorithms lies in the design of exploitation-vs-exploration mechanism. 

For UCB-based algorithms \citep{auer2002finite,audibert2009exploration,garivier2011kl,kaufmann2012bayesian,lattimore2018refining}, a UCB index that captures an exploitation component and an exploration component  typically is computed for each arm, and then the learning agent makes a decision based on all the UCB indices. Take the vanilla UCB1 \citep{auer2002finite}  for example. The UCB index is the sum of the empirical mean (exploitation component) and the square root of the multiplicative inverse of the number of pulls (exploration component). For Thompson Sampling-based algorithms \citep{kaufmann2012thompson,agrawal2017near,jin2021mots,jin2022finite,jin2023thompson,bian2022maillard,Opt2023,hu2020problem,hu2023optimistic}, instead of computing a deterministic UCB index for each arm, the learning agent first constructs a data-dependent distribution whose mean takes into account the empirical mean and whose variance takes into account the multiplicative inverse of the number of pulls. Note that the mean of the data-dependent distribution contributes to the exploitation while the spread of the data-dependent distribution contributes to the exploration. 
After constructing 
the data-dependent distribution, the learning agent
 draws a random sample  for each arm and makes a decision based on all the drawn samples. If the data-dependent distribution is designed to be a posterior distribution with a conceptually assumed prior distribution and reward likelihood function \citep{kaufmann2012thompson,agrawal2017near}, it recovers the original Thompson Sampling, one of the oldest Bayesian-inspired randomized learning algorithms. There are two original versions of Thompson Sampling based on the conceptually assumed prior distributions and reward likelihood functions. \cite{kaufmann2012thompson,agrawal2017near} derived instance-dependent bounds for the version where uniform distribution $\text{Beta}\left(1,1 \right)$ is assumed to be the prior distribution that models mean reward, that is, $\mu_j \sim \text{Beta}\left(1,1 \right)$, and the probability mass function of $\text{Bernoulli}(\mu_j)$ is assumed to be the reward likelihood function. \cite{agrawal2017near,Opt2023} derived instance-dependent bounds for the version where a zero-mean Gaussian distribution with infinite variance is assumed to be  the prior distribution, that is, $\mu_j \sim \mathcal{N} \left(0, \infty \right)$, and the probability dense function of $\mathcal{N} (\mu_j, 1)$ is assumed to be the reward likelihood function. Different from the UCB-based and Thompson Sampling-based algorithms that natively are anytime and operate on a per-round basis, elimination-based algorithms usually run in epochs. At the end of each epoch, such algorithms eliminate arms with poor empirical means. 
 
 For each of the aforementioned families of learning algorithms, there are algorithms \citep{auer2002finite,audibert2009exploration,auer2010ucb,agrawal2017near,jin2022finite,bian2022maillard,lattimore2018refining,kaufmann2012bayesian,Opt2023} that achieve instance-dependent regret bounds
that are sub-UCB, meaning a regret bound of the $\sum_{j: \Delta_j > 0}O \left(  \frac{\log(T)}{\Delta_j} \right)$ form.
In some cases \citep{garivier2011kl,honda2015non,agrawal2017near,kaufmann2012bayesian}, more refined problem-dependent regret bounds $\sum_{j: \Delta_j > 0} \frac{(1+\kappa)\ln(T) \Delta_j}{d_{\text{KL}}\left( \mu_j, \mu_1 \right)} + O \left(\frac{1}{\kappa}\right)$ have been shown for $\kappa > 0$ a universal constant that can be arbitrarily small and $d_{\text{KL}}\left(a,b\right)$ the KL-divergence between two Bernoulli distributions with parameters $0 < a, b < 1$.

\paragraph{Private learning.}

As shown in Proposition~2.1 in \citet{dwork2014algorithmic},  differential privacy is invariant to post-processing.  
In other words, if a learning algorithm takes the output of an $\epsilon$-differentially private  algorithm as input, then the output of this learning algorithm itself is also $\epsilon$-differentially private. In designing stochastic bandit algorithms with differential privacy, if the inner algorithm to compute the empirical mean is $\epsilon$-differentially private, then following from the post-processing property, we can claim the bandit  algorithm itself is  $\epsilon$-differentially private. This framework has indeed been used in the design of differentially private algorithms in previous work \citep{mishra2015nearly,tossou2016algorithms,sajed2019optimal,chen2020locally,azize2022privacy,chandak2023differentially,hu2022near}.

\citet{mishra2015nearly} presented the first  UCB-based and Thompson Sampling-based algorithms for differentially private stochastic bandits.  However, the regret bounds they derived are far from the ``private sub-UCB'' $\Omega \left( \sum_{j: \Delta_j > 0} \frac{\log(T)}{\Delta_j} + \frac{K \log(T)}{\epsilon} \right)$
instance-dependent regret lower bound due to
\citet{shariff2018differentially}.  The key reason resulting in the suboptimality is the usage of  the $T$-bounded Binary Mechanism \citep{dwork2010differential,chan2011private} to inject  random noise to mask the true empirical mean  of each arm.\footnote{\citet{dwork2010differential} call it the Tree-based Mechanism, but the core idea is identical.} Also, since  their algorithms need to know the time horizon $T$ in advance to calibrate the distribution of the needed noise, they cannot be anytime.
Later,  \citet{sajed2019optimal} proposed DP-SE, an optimal elimination-style algorithm.  The key idea in achieving optimality is to use fresh observations to compute the differentially private empirical means, thus minimizing the number of noise variables needed and the variance of the noise distribution. 
Our algorithms, Anytime-Lazy-UCB and Lazy-DP-TS, are the first optimal UCB-based and Thompson Sampling-based algorithms for differentially private stochastic bandits.

\subsection{Recent Advances in the Full Information Setting}

\paragraph{Non-private learning.}

For the non-private stochastic full information setting, the balance between exploitation-exploration is not necessary as the environment reveals the complete reward vector regardless of the action taken, i.e., pure exploitation is enough. \cite{vanerven2014follow} showed that FTL achieves the optimal $O \left(\log(K)/\Deltamin \right)$ regret bound. 
In addition, Decreasing Hedge --- which involves running Hedge \citep{freund1997decision} with a time-varying learning rate $\eta_t$ of order $\sqrt{\log(K)/t}$ --- enjoys an instance-dependent regret bound of the same order \citep{mourtada2019optimality}. The interest of the latter algorithm is that it simultaneously enjoys $O(\sqrt{T \log K})$ regret if the rewards are generated by an adaptive adversary.

\paragraph{Private learning.}

Let us recount the developments for the private stochastic full information setting. 
 \citet{jain2012differentially, guha2013nearly, agarwal2017price} derived regret bounds with adversarial losses and until very recently, the best known regret bound (in terms of the problem of prediction with expert advice) was $O \left( \sqrt{T\log(K)} + \frac{K\log(K)\log^2(T)}{\epsilon} \right)$ \citep{agarwal2017price}. Note that the term involving  $\epsilon$ is at least linear in $K$. 
 Very recently, \citet{asi2022private} showed an instance-\emph{independent} regret bound of order $O \left( \sqrt{T \log K} + \frac{\log(K) \log(T)}{\epsilon} \right)$.
 In light of the optimal regret in the non-private setting, it is natural to wonder whether any private learning algorithm can achieve $O \left(\frac{\log(K)}{ \Delta} + \frac{\log(K)}{ \epsilon} \right)$ regret. In this work, we  show the first instance-dependent regret lower bound stating that any $\epsilon$-differentially private 
 learning algorithm must incur $\Omega \left(\log(K)/ \min \left\{ \epsilon, \Delta \right\}  \right)$ regret. We also show the first minimax regret lower bound stating that any $\epsilon$-differentially private 
 learning algorithm must incur $\Omega \left(\sqrt{T \log(K)} + \log(K)/\epsilon \right)$ regret.
Our proposed algorithm, RNM-FTNL, achieves an instance-dependent regret bound of order 
$O \left( \frac{ \log K}{ \Deltamin} + \frac{\log(K) \log \left( \Delta_{\max}/\max\left\{ \Deltamin, \epsilon, 1/T \right\} \right)}{\epsilon} \right)$
as well as the same instance-independent regret bound as \citet{asi2022private}.\footnote{
A previous version of this work \citep[arXiv:2102.07929v1]{hu2021optimal} contained exactly the same algorithm, FTNL, as appears in the present paper. However, in that version we claimed a regret bound of order 
$O \left( \frac{\log K}{\min\{\Deltamin, \epsilon\}} \right)$ 
based on a proof with a mistake at one step; for details, see the discussion after the proof of Corollary~\ref{cor:ftnl-indep-regret}.} Therefore, up to a $\log(T)$ factor: \emph{(i)} our algorithm is instance-dependent optimal; and \emph{(ii)} 
both our algorithm and that of \cite{asi2022private} are minimax optimal. 

\section{Optimal Private Algorithms for Bandit Setting} \label{sec:bandit}

In this section, we present our proposed  learning algorithms.  Anytime-Lazy-UCB is presented in Section~\ref{sec: lazy-ucb} and  Lazy-DP-TS is presented in Section~\ref{sec: lazy-ts}.
Both of our proposed learning algorithms  are motivated by the differentially private instance-dependent regret lower bound shown in Corollary~15 of \cite{shariff2018differentially} and DP-SE of \cite{sajed2019optimal}. These bounds show that  a good learning algorithm can pull a suboptimal arm $j$  $O \left(\frac{\log(T)}{\Delta_j \cdot \min \left\{\Delta_j, \epsilon\right\}}\right)$ times.
Take a simple $K$-armed bandit instance with  identical gap $\Delta$ for example. If we knew $\Delta$ and $T$ in advance,  we could pull each arm $j$ exactly $O \left(\frac{\log(T)}{\Delta \cdot \min \left\{\Delta, \epsilon\right\}}\right)$ times to have a single batch of data associated with arm $j$.
Then, we add a $\text{Lap}\left(1/\epsilon \right)$  noise to the aggregated reward of this batch of data and compute the differentially private empirical mean of each arm $j$.\footnote{$\text{Lap} (b)$ denotes a random variable drawn from a Laplace distribution centered at $0$ with scale $b$. The probability density function of $\text{Lap} (b)$ can be found in Fact~\ref{pdf of lap} in the appendix.} With  all the differentially private empirical means  in hand, we can either construct upper confidence bounds (UCB-based) or draw random samples from data-dependent distributions (Thompson Sampling-based) to make future decisions. Since  now all suboptimal arms have been sufficiently observed,
the chance to pull any suboptimal arm in the future is very low.

\paragraph{Importance of Doubling-Trick.} As in reality, we cannot assume all the suboptimal arms have the same mean reward gap nor we know  $T$ in advance, we can use \emph{doubling trick}  to estimate $O \left( \frac{\log (T)}{\Delta_j \cdot \min \{\Delta_j, \epsilon\}} \right)$ as a whole instead of estimating $\Delta_j$ and $T$ separately. The estimation can be achieved by 
constructing  arm-specific epochs and processing data associated with arm $j$ in batches.  For each arm $j$, in the $r{\nth} $ epoch, our goal is to collect $2^{r}$ observations  to form a batch. Then, we add a $\text{Lap} (1/\epsilon)$ noise to the aggregated reward of these $2^{r}$ observations to compute the differentially private empirical mean of this batch of data. Once the batch size for some epoch hits $O \left( \frac{\log (T)}{\Delta_j \cdot \min \{\Delta_j, \epsilon\}} \right)$, the probability of pulling this suboptimal arm $j$ in the future is very low.

\paragraph{Notation for Algorithms Presented in Section~\ref{sec:bandit}.}
Let $r_j(t-1)$ denote the index of the most-recently processed epoch by the end of round $t-1$. Here, we say an epoch is processed if all the observations obtained in that epoch are used to update the (differentially private) empirical mean.
Based on the definition of $r_j(t-1)$, we know the number of observations $O_j(t-1)$ that have been  used to compute the (differentially private) empirical mean is exactly $ 2^{r_j(t-1)}$. Also,  we know that the learning algorithm has not had the chance to process the batch data collected in the $\left(r_j(t-1)+1\right){\nth}$ epoch. 
Let $\widehat{\mu}_{j, O_j(t-1)}$ denote the empirical mean based on the batch data collected in the $r_j(t-1){\nth}$ epoch. Let $\widetilde{\mu}_{j, O_j(t-1)} := \widehat{\mu}_{j, O_j(t-1)} + \frac{\text{Lap}(1/\epsilon)}{O_j(t-1)}$ denote the differentially private empirical mean.

\subsection{Anytime-Lazy-UCB}\label{sec: lazy-ucb}

 Algorithm~\ref{Optimal DP-UCB} presents the pseudo-code of Anytime-Lazy-UCB. It is an anytime learning algorithm as it does not require to input time horizon $T$. Line~\ref{ini}  is to initialize  the learning algorithm. We pull each arm  once to initialize $O_j(t-1) $, $r_j(t-1)$, and $\widetilde{\mu}_{j, O_j(t-1)}$.  Also, we initialize the number of unprocessed observations    $N_j = 0$ and the aggregated reward among unprocessed observations $ \Phi_j = 0$.
  Starting from round $t = K+1$, the learning algorithm first constructs the (differentially private) upper confidence bounds $\overline{\mu}_{j, O_j(t-1)}(t) := \widetilde{\mu}_{j, O_j(t-1)} + \sqrt{3\ln(t)/O_j(t-1)} + 3\ln(t)/(\epsilon \cdot O_j(t-1))$ for all $j \in [K]$  and then pulls the arm $J_t = \mathop{\arg\max}_{j \in [K]} \overline{\mu}_{j, O_j(t-1)}(t)$ with the highest differentially private upper confidence bound (Line~\ref{construct ucb}). Note that the added bonus  term $3\ln(t)/(\epsilon \cdot O_j(t-1))$   is to construct an upper confidence bound for the injected Laplace noise.\footnote{Our non-private upper confidence bound is based on UCB1 policy \citep{auer2002finite}. We can also construct upper confidence bounds based on other  policies such as UCB-V of \cite{audibert2009exploration} and KL-UCB of \cite{garivier2011kl}. AdaP-KLUCB of \cite{azize2022privacy} is based on KL-UCB.} Lines~\ref{process start} to \ref{process end}  present how to process the obtained observations. 
We increment the counter $N_{J_t}$ by one and update the corresponding aggregated reward $\Phi_{J_t}$. Then, we check whether it is the right time to update the differentially private empirical mean of the pulled arm $J_t$. If $N_{J_t}$  hits $2^{r_{J_t}+1}$,  we update the differentially private empirical mean of arm $J_t$ as $\left(\Phi_{J_t} + \text{Lap}\left(1/\epsilon \right)\right)/N_{J_t}$. Since now all the observations have been processed, we increment the index of the processed epoch by one and reset $N_{J_t}$ and $\Phi_{J_t}$ to process the future observations of the pulled arm $J_t$.
\begin{algorithm}[H]
    \centering
    \caption{Anytime-Lazy-UCB}
	\label{Optimal DP-UCB}
    \footnotesize
	\begin{algorithmic}[1]
	\STATE {\bf{Input}:} Arm set $\mathcal{A}$ and privacy parameter $\epsilon$

	\STATE \label{ini}{\bf{Initialization}:} For each arm $j \in [K]$, 
		 pull it once to initialize  $O_{j} \leftarrow 1 , r_j \leftarrow 0$,      $\widetilde{\mu}_{j, O_{j}}  \leftarrow X_{j}  + \text{Lap} \left( \frac{1}{\epsilon} \right)$,  $N_{j} \leftarrow 0$, and   $\Phi_{j} \leftarrow 0$

	\FOR {$t = K+1, K+2, \dotsc$} \label{post 7}

\STATE \label{construct ucb} Construct  
$\overline{\mu}_{j, O_j}(t) = \widetilde{\mu}_{j, O_j}  +\sqrt{\frac{3\ln(t)}{O_j}} + \frac{3\ln(t)}{\epsilon \cdot O_j }$ for all $j$ and pull arm  $J_t \in \mathop{\arg\max}_{j \in [K]} \overline{\mu}_{j, O_j}(t)$

\STATE Update $N_{J_t} \leftarrow N_{J_t}+ 1$ and $\Phi_{J_t} \leftarrow \Phi_{J_t}+ X_{J_t}(t)$    \label{process start}

	\IF {$N_{J_t} $ hits $2^{r_{J_t}+1}$ }
	\STATE Update $O_{J_t} \leftarrow 2^{r_{J_t}+1}$, $\widetilde{\mu}_{J_t, O_{J_t}} \leftarrow \frac{ \Phi_{J_t} +\text{Lap} \left(\frac{1}{\epsilon}\right)}{O_{J_t}}$, and $r_{J_t} \leftarrow r_{J_t} + 1$;  Reset $N_{J_t} \leftarrow 0, \Phi_{J_t} \leftarrow 0$.
	\ENDIF \label{process end}
	\ENDFOR	
\end{algorithmic}
\end{algorithm}
\paragraph{Remarks  for Algorithm~\ref{Optimal DP-UCB}.} 
(a) The number of observations  used to compute the differentially private empirical mean doubles each time, i.e., $O_j(t-1)$ takes values from $2^{r_j}, r_j \ge 0$. (b) The number of  noise  variables included in the differentially  private empirical means is always 1 and it is drawn from $\text{Lap}\left(1/\epsilon 
\right)$. (c) An observation can only be used at most once. Once it is processed, it will be discarded by the learning algorithm.

\paragraph{Anytime-Lazy-UCB vs AdaP-UCB of \citet{azize2022privacy}.} 
Essentially, AdaP-UCB in \citet{azize2022privacy} is equivalent to our Anytime-Lazy-UCB with the key modifications summarized as follows. 
(a) Another way  to implement the doubling trick is to use the accumulated amount of observations, as used in the Logarithmic Mechanism of \cite{chan2011private}.  That is, the differentially private aggregated reward is computed when the number of observations is accumulated to exactly $2^r, r \ge 0$. 
(b) Another difference is Anytime-Lazy-UCB injects a $\text{Lap}(1/\epsilon)$ noise to  the aggregated reward, whereas AdaP-UCB injects a noise to the empirical mean.  However, essentially, they are the same.

We now present a privacy guarantee and regret guarantees for Algorithm~\ref{Optimal DP-UCB}. 
\begin{theorem}
Algorithm~\ref{Optimal DP-UCB} is $\epsilon$-differentially private.
\label{DP: Optimal DP-UCB}
\end{theorem}
\begin{proof-sketch}
Intuitively, differential privacy holds because if two neighbouring reward vector sequences only differ in a single round $t$, then this difference can only be witnessed by the pulled arm  $J_t$ in round $t$. Ultimately, the affected observation $X_{J_t}(t)$ will be used only once via a noisy sum involving   $\text{Lap} \left( 1/\epsilon \right) $ noise. 
Hence, $\epsilon$-differential privacy holds. For completeness, we give a full, mathematical proof in Appendix~\ref{app: dp}.
\end{proof-sketch}
\begin{theorem}
The instance-dependent regret  of Algorithm~\ref{Optimal DP-UCB} is 
$O \left( \sum_{j: \Delta_j > 0}    \log (T)/ \min \{\Delta_j, \epsilon\}  \right)$ and the instance-independent regret  of Algorithm~\ref{Optimal DP-UCB} is 
$O \left( \sqrt{KT\log(T)} + K\log(T)/\epsilon \right)$.
\label{Regret: Optimal DP-UCB}
\end{theorem}
Anytime-Lazy-UCB is instance-dependent optimal due to the $\Omega \left( \max \left\{ \sum_{j: \Delta_j > 0}\frac{ \log (T)}{\Delta_j} , \frac{K \log (T)}{\epsilon} \right\} \right)$ instance-dependent regret lower bound  \citep{shariff2018differentially}. Anytime-Lazy-UCB is minimax optimal up to an extra $\log(T)$ factor due to the $\Omega \left(\max \left\{\sqrt{KT}, \frac{K}{\epsilon} \right\} \right)$ minimax regret lower  bound \citep{azize2022privacy}.

Now, we sketch the proof for the instance-dependent regret bound  and defer the full proof for Theorem~\ref{Regret: Optimal DP-UCB} to Appendix~\ref{app: Lazy-DP-UCB}.

\begin{proof-sketch}(of Theorem~\ref{Regret: Optimal DP-UCB})
Recall $r_j(T)$ is the index of the most-recent processed epoch  by the end of round $T$. 
Let $n_j(T)$ denote the number of pulls of arm $j$ by the end of round $T$. Note both $r_j(T)$ and $n_j(T)$ are random, and we have $n_j(T) \le \sum_{q=0}^{r_j(T)+1}2^q$.
Instead of upper bounding $\mathbb{E}\left[n_j(T)\right]$, we upper bound  $\mathbb{E}\left[\sum_{q=0}^{r_j(T)+1}2^q\right]$.  
Let $d_j := \left\lceil \log \left(   \frac{24 \ln (T)}{\Delta_j \cdot \min \{\Delta_j, \epsilon\}}  \right) \right\rceil $. Intuitively,  $2^{d_j}$ is exactly the ``differentially private sample complexity'' of a suboptimal arm $j$.
Let $\omega_j^{(r)} :=   \sum_{q = 0}^{r} 2^{q} = 2^{r+1} - 1 $ count the  number of pulls  needed for  the $r{\nth}$ epoch to be processed.
Then, we have  $ \mathbb{E} \left[ n_j(T) \right]   
     \le  \omega_j^{(d_j+1)} + \sum_{s= d_j+1}^{\log(T)}\mathbb{P} \left\{r_j(T) = s \right\}  \left(\sum_{q = d_j + 1}^{s+1} 2^q   \right)
  \le \omega_j^{(d_j+1)} + 
 \sum_{s= d_j+1}^{\log(T)}\mathbb{P} \left\{r_j(T) = s \right\}   \omega_j^{(s+1)}$. Note that the value of $r_j(T)$ can be at most $\log (T)$ based on doubling-trick.
When $s \ge d_j+1$, we know  the number of observations  used to compute the differentially private empirical mean is lower bounded by $2^{d_j} = O \left( \frac{\ln (T)}{\Delta_j \cdot \min \{\Delta_j, \epsilon\}} \right) $, the ``differentially private sample complexity''. 
This amount of observations ensures that the probability that arm $j$ is pulled in the future is very low. We formalize this intuition in Lemma~\ref{Tomorrow2} stating that $\mathbb{P} \left\{ r_j(T) = s  \right\} \cdot \omega_j^{(s+1)} \le O(1)$, which concludes the proof. 
\end{proof-sketch}

\subsection{Lazy-DP-TS} \label{sec: lazy-ts}

Before presenting Lazy-DP-TS, we first recap  one of the original versions of  Thompson Sampling: Thompson Sampling using Beta Priors for Bernoulli bandits \citep{agrawal2017near}. Algorithm~\ref{Non-private TS} restates this classical algorithm.  
Different from   UCB-based algorithms where the exploitation-exploration trade-off is achieved by  the constructed upper confidence bounds, Thompson Sampling simply maintains a posterior distribution over the mean reward of each arm based on  some simple reward likelihood function and its conjugate prior.
In Algorithm~\ref{Non-private TS} the used conjugate pair is a  Bernoulli likelihood and a $\text{Beta}(1,1)$ prior, i.e., the uniform distribution.
Then, the posterior distribution can be computed as $\text{Beta} \left(\widehat{\mu}_{j, O_j} \cdot O_j +1, \left(1-\widehat{\mu}_{j, O_j}   \right)\cdot O_j +1\right)$, where $O_j$ is the total number of observations by the end of round $t-1$ and $\widehat{\mu}_{j, O_j}$ is the empirical mean. In each round $t$, the learning agent draws a posterior sample $\theta_j(t)$ for each arm and pulls the arm $J_t = \mathop{\arg\max}_{j \in [K]}\theta_j(t)$ with the highest sample value.

\begin{algorithm}[H]
	\caption{Thompson Sampling using Beta Priors \citep{agrawal2017near}}
	\label{Non-private TS}
	    \footnotesize
	\begin{algorithmic}[1]
	\STATE {\bf{Input}:} Arm set $\mathcal{A}$ 
	\STATE {\bf{Initialization}:} Set $O_j \leftarrow 0$, $\widehat{\mu}_{j, O_j} \leftarrow 0$
	\FOR {$t = 1,2, \dotsc$} 

\STATE For each arm $j \in [K]$, draw $\theta_j(t) \sim \text{Beta}\left(\widehat{\mu}_{j, O_j} \cdot O_j + 1, \left(1-\widehat{\mu}_{j, O_j} \right)\cdot O_j + 1\right)$ 

\STATE Pull  $J_t \in \mathop{\arg\max}_{j \in [K]} \theta_j(t)$ 
 and update $O_{J_t} \leftarrow O_{J_t} + 1$ and the empirical mean $\widehat{\mu}_{J_t, O_{J_t}} $.
\ENDFOR	
\end{algorithmic}
\end{algorithm}
Algorithm~\ref{Optimal DP-TS} presents the pseudo-code of Lazy-DP-TS. Like Algorithm~\ref{Non-private TS}, it is still  anytime. The initialization (Line~\ref{ini-2}) is the same as the one used in Algorithm~\ref{Optimal DP-UCB}.    Starting from round $t = K+1$, we first compute the parameter $\overline{\mu}_{j, O_j(t-1)}(t) :=\text{Clip}_{[0,1]} \left( \widetilde{\mu}_{j, O_j(t-1)}  +  \frac{3\log(t)}{\epsilon \cdot O_j(t-1) } \right) \in [0,1]$ for the Beta posterior distribution, where $\text{Clip}_{[0,1]}(x) = \bm{1} \left\{x \in [0,1] \right\}x + \bm{1} \left\{x > 1 \right\}$ denotes a  clipping function that maps any $x \in \mathbb{R}$ to a real number in $[0,1]$. The purpose of 
adding a bonus term $\frac{3\log(t)}{\epsilon \cdot O_j(t-1) }$ to the differentially private empirical mean  $\widetilde{\mu}_{j, O_j(t-1)}$ is to construct an upper confidence bound on the empirical mean $\widetilde{\mu}_{j, O_j(t-1)}$, and the  clipping is to ensure that the   Beta distribution parameters are valid, that is, non-negative. 
Next, we draw a  sample 
$\theta_j(t) \sim \text{Beta} \left(\overline{\mu}_{j, O_j(t-1)}(t) \cdot O_j(t-1)+1, \left(1-\overline{\mu}_{j, O_j(t-1)}(t) \right) \cdot O_j(t-1)+1 \right)$ and pull the arm $J_t = \mathop{\arg\max}_{j \in \mathcal{A}}\theta_j(t)$ with the highest  sample value. Finally, we process the obtained observations exactly as we did in Algorithm~\ref{Optimal DP-UCB}.
\begin{algorithm}[H]
  \centering
	\caption{Lazy-DP-TS}
 \footnotesize
	\label{Optimal DP-TS}
	\begin{algorithmic}[1]
	\STATE {\bf{Input}:} Arm set $\mathcal{A}$ and privacy parameter $\epsilon$
\STATE \label{ini-2}{\bf{Initialization}:} For each arm $j$,  pull it once to initialize  $O_{j} \leftarrow 1 $, $r_j \leftarrow 0$,    $\widetilde{\mu}_{j, O_{j}}  \leftarrow X_{j}  + \text{Lap} \left( \frac{1}{\epsilon} \right)$, $N_{j} \leftarrow 0$, and   $\Phi_{j} \leftarrow 0$
	
\FOR {$t = K+1, K+2, \dotsc$}

\STATE \label{post 11}Draw $\theta_j(t) \sim \text{Beta}\left(\overline{\mu}_{j, O_j} \cdot O_j + 1, \left(1-\overline{\mu}_{j, O_j} \right)\cdot O_j + 1\right)$, where
$\overline{\mu}_{j, O_j} =\text{Clip}_{[0,1]} \left( \widetilde{\mu}_{j, O_j}  +  \frac{3\log(t)}{\epsilon \cdot O_j } \right)$

\STATE Pull arm  $J_t \in \mathop{\arg\max}_{j \in [K]} \theta_j(t)$ \label{post 22}

\STATE \label{process 1} Update $N_{J_t} \leftarrow N_{J_t}+1$ and $\Phi_{J_t} \leftarrow \Phi_{J_t} + X_{J_t}(t)$

\IF {$N_{J_t} $ hits $2^{r_{J_t}+1}$ }
	\STATE Update $O_{J_t} \leftarrow 2^{r_{J_t}+1}$, $\widetilde{\mu}_{J_t, O_{J_t}} \leftarrow \frac{ \Phi_{J_t} +\text{Lap} \left(\frac{1}{\epsilon}\right)}{O_{J_t}}$, and $r_{J_t} \leftarrow r_{J_t} + 1$;  Reset $N_{J_t} \leftarrow 0, \Phi_{J_t} \leftarrow 0$.
	\ENDIF \label{process 2}
\ENDFOR	
\end{algorithmic}
\end{algorithm}

We now present a privacy guarantee and regret guarantees  for Algorithm~\ref{Optimal DP-TS}.

\begin{theorem}
Algorithm~\ref{Optimal DP-TS} is $\epsilon$-differentially private.
\label{DP: Optimal DP-TS}
\end{theorem}
\begin{proof}
 Lines \ref{process 1} to \ref{process 2} in Algorithm~\ref{Optimal DP-TS} are for computing the differentially private empirical mean and Lines~\ref{post 11} to \ref{post 22} can be viewed as post-processing. Now, we only need to show that the inner algorithm (Lines \ref{process 1} to \ref{process 2}) to compute the empirical mean is $\epsilon$-differentially private.
Suppose reward sequences ${X}$ and ${X}'$ differ in round $h$, i.e., the reward vectors $X_h = \left(X_1(h), \dotsc, X_K(h) \right)$ and $X'_h = \left(X'_1(h), \dotsc, X'_K(h) \right)$  are not the same. The changing from $X_h$ to $X'_h$ can only impact arm $J_h$. Let $J_h = j$. 
Since arm $j$'s differentially private means are always based on fresh observations, the changing from $X_h$ to $X'_h$ can only impact the differentially private aggregated reward  of arm $j$ once and by at most 1. By adding a noise  variable drawn from $\text{Lap} \left(\frac{1}{\epsilon} \right)$ to $\Phi_{j}$, from Theorem~3.6 in \citet{dwork2014algorithmic}, we know that the inner algorithm is $\epsilon$-differentially private, which concludes the proof.
\end{proof}
\begin{theorem}
The instance-dependent regret  of Algorithm~\ref{Optimal DP-TS}  is 
$  O \left(  \sum_{j: \Delta_j > 0}
   \log(T)/\min \left\{\Delta_j, \epsilon\right\}\right) $ and the instance-independent regret  of Algorithm~\ref{Optimal DP-TS}  is $O \left( \sqrt{KT\log(T)} + K\log(T) /\epsilon\right)$.
\label{Regret: Lazy-DP-TS}
\end{theorem}
We now sketch the instance-dependent regret bound proof. The full proof   is deferred to Appendix~\ref{app: ts}. The proof for the instance-independent regret bound  is the same as the one for Anytime-Lazy-UCB shown in Appendix~\ref{app: isolated 2}.
Let $\mathcal{F}_{t-1}  $ collect all the history information by the end of round $t-1$, consisting of the pulled arms, the obtained rewards, and the added noise. Define $\mathcal{F}_0 = \left\{ \right\}$. We first define two high probability events  motivated by the concentration inequalities shown in Facts~\ref{Hoeffding} and \ref{fact 1}.
Define event $C_j(t-1)$ as $ \left\{\left|\widehat{\mu}_{j, O_j(t-1)} - \mu_j \right| \le \sqrt{3\log(t)/O_j(t-1)}\right\}$. Define  $G_j(t-1)$ as the event that the confidence interval of the Laplace random variable holds, i.e., $G_j(t-1) := \left\{\left|\widetilde{\mu}_{j, O_j(t-1)} - \widehat{\mu}_{j, O_j(t-1)} \right|\le 3\log(t)/\left(\epsilon \cdot O_j(t-1)\right) \right\}$.  For each suboptimal arm $j$, we let $y_j := \mu_1 - \Delta_j/6$ and  define event $E_j^{\theta}(t) := \left\{\theta_j(t) \le y_j  \right\}$. Let $\overline{C_j(t-1)}$, $\overline{G_j(t-1)}$ and $\overline{E_j^{\theta}(t)}$ denote the complementary events of $C_j(t-1)$, $G_j(t-1)$, and $E_j^{\theta}(t) $, respectively.

\

\begin{proof-sketch}(of Theorem~\ref{Regret: Lazy-DP-TS})
We directly upper bound the expected number of pulls $\mathbb{E} \left[n_j(T) \right]$ a suboptimal arm $j$ by the end of round $T$. We decompose $\mathbb{E} \left[n_j(T) \right]$ as
\begin{equation}
    \begin{array}{lll}
 \mathbb{E} \left[n_j(T) \right] & = & \sum\limits_{t = K+1}^{T}\mathbb{E} \left[ \bm{1} \left\{ J_t = j \right\}  \right] +1 \\
 
 &   \le & \underbrace{\sum\limits_{t = K+1}^{T}\mathbb{P}  \left\{J_t = j, C_j(t-1), G_j(t-1), \overline{E_j^{\theta}(t)} \right\}}_{=:\omega_1} 
   +  \underbrace{\sum\limits_{t = K+1}^{T} \mathbb{P}  \left\{J_t = j,   E_j^{\theta}(t), G_1(t-1) \right\}}_{=:\omega_2} \\
   & + & \underbrace{\sum\limits_{t = K+1}^{T} \mathbb{P}   \left\{ \overline{C_j(t-1)} \right\} +  \mathbb{P} \left\{ \overline{G_j(t-1)} \right\} +  \mathbb{P} \left\{ \overline{G_1(t-1)} \right\}}_{=: \omega_3} +1\quad.
    \end{array}
    \label{cloud 22}
\end{equation}
Since $\overline{C_j(t-1)}, \overline{G_j(t-1)}$ and $\overline{G_1(t-1)}$ are  low probability events, we have  $\omega_3 = O(1)$.

\paragraph{Upper bounding $\omega_1$.} We partition all $T$ rounds into multiple intervals based on whether arm $j$'s differentially private empirical mean is updated. Let $\lambda^{(j)}_{r}$ denote the round  by the end of which  arm $j$'s differentially private empirical mean will be updated based on the $2^{r}$ 
  observations obtained in the $r{\nth}$ epoch.
Let $d_j := \left\lceil \log \left( \frac{72 \log(T)}{\Delta_j \cdot \min \left\{\epsilon, \Delta_j \right\}} \right) \right\rceil$. 
For all the rounds up to (and including) round $\lambda^{(j)}_{d_j}$,  the total number of pulls is at most $\sum_{r=0}^{d_j} 2^r = O \left( \frac{\log(T)}{\Delta_j \cdot \min \left\{\epsilon, \Delta_j \right\}} \right)$. For all the rounds starting from round  $\lambda^{(j)}_{d_j}+1$, only concentration bounds are needed. Intuitively, if both of the high probability events $C_j(t-1)$ and $G_j(t-1)$ hold, and the number of observations to compute the differentially private empirical mean is enough, i.e., greater than $2^{d_j}$,  the probability to draw $\theta_j(t)$ greater than $y_j$ is only of order  $1/T$. Lemma~\ref{lazy dp-ts high term2} in Appendix~\ref{app: ts}  states that $\omega_1 = O \left( \frac{\log(T)}{\Delta_j \cdot \min \left\{\epsilon, \Delta_j  \right\}}\right)$.

\paragraph{Upper bounding $\omega_2$.} It is  challenging  to upper bound $\omega_2$. We  do it in  four steps. 

\textbf{Step 1:} Similar to Lemma~2.8 in  \citet{agrawal2017near}, we develop Lemma~\ref{Martingale event} to link the probability to pull a suboptimal arm $j$ to the probability of pulling the optimal arm $1$. Let $\mathbb{P}_{t-1}\left\{ \cdot \right\} := \mathbb{P}\left\{ \cdot \mid \mathcal{F}_{t-1} = F_{t-1}\right\}$.
\begin{lemma}\label{Martingale event}
For any suboptimal $j$ and any instantiation $F_{t-1}$ of $\mathcal{F}_{t-1}$, we have

$   \mathbb{P}_{t-1}   \left\{J_t = j ,  E^{\theta}_j(t) , G_1(t-1) \right\}
    \le  \frac{\mathbb{P}_{t-1}   \left\{\theta_1(t) \le y_j   \right\}  }{\mathbb{P}_{t-1}   \left\{\theta_1(t) > y_j    \right\} } \cdot \mathbb{P}_{t-1}   \left\{J_t = 1 ,   E^{\theta}_j(t), G_1(t-1)  \right\}$.
\end{lemma}
After using Lemma~\ref{Martingale event}, we have   $\omega_2 \le \sum_{t = K+1}^{T} \mathbb{E} \left[ \frac{\mathbb{P} \left\{\theta_1(t) \le y_j \mid \mathcal{F}_{t-1}\right\}}{\mathbb{P} \left\{\theta_1(t) > y_j \mid \mathcal{F}_{t-1} \right\}} \bm{1}\left\{J_t = 1,    G_1(t-1) \right\} \right]$. 

\textbf{Step 2:} We partition all $T$ rounds into multiple intervals based on whether the differentially private empirical mean of the optimal arm $1$ is updated. Let $\lambda^{(1)}_r$ be the round such that by the end of this round, $2^r$ observations obtained in the $r{\nth}$ epoch associated with the optimal arm $1$ will be used to compute the differentially private empirical mean. After the partition, we have
$\omega_2 \le \sum_{r=0}^{\log(T)}\mathbb{E} \left[\sum_{t = \lambda^{(1)}_r +1}^{\lambda^{(1)}_{r+1}} \frac{\mathbb{P} \left\{\theta_1(t) \le y_j \mid \mathcal{F}_{t-1}\right\}}{\mathbb{P} \left\{\theta_1(t) > y_j \mid \mathcal{F}_{t-1} \right\}} \bm{1}\left\{J_t = 1,    G_1(t-1) \right\} \right]$.

\textbf{Step 3:} Despite $\widetilde{\mu}_{1,O_1(t-1)}$ staying the same in all rounds $t \in \left\{\lambda^{(1)}_{r}+1, \dotsc, \lambda^{(1)}_{r+1} \right\}$, the posterior distributions for $\theta_1(t)$  still change due to the added bonus term. The change of the  distributions  implies  the change of the coefficient $\frac{\mathbb{P} \left\{\theta_1(t) \le y_j \mid \mathcal{F}_{t-1}\right\}}{\mathbb{P} \left\{\theta_1(t) > y_j \mid \mathcal{F}_{t-1} \right\}}$.
To tackle this challenge we construct another Beta distribution and use the property of first-order stochastic dominance.
Let $\theta'_{1, O_1(t-1)} \sim \text{Beta} \left(\widehat{\mu}_{1,O_1(t-1)} \cdot O_1(t-1) +1, (1-\widehat{\mu}_{1,O_1(t-1)}) \cdot O_1(t-1) +1 \right)$ denote a random variable drawn from a Beta distribution based on the (non-private) empirical mean $\widehat{\mu}_{1,O_1(t-1)}$. For any instantiation $F_{t-1}$ of $\mathcal{F}_{t-1}$ such that event $G_1(t-1)$ is true, we have  $    \frac{\mathbb{P}_{t-1}   \left\{\theta_1(t) \le y_j   \right\}  }{\mathbb{P}_{t-1}   \left\{\theta_1(t) > y_j  \right\} }  \le  \frac{\mathbb{P}_{t-1}   \left\{\theta'_{1, O_1(t-1)}\le y_j   \right\}  }{\mathbb{P}_{t-1}   \left\{\theta'_{1, O_1(t-1)} > y_j  \right\} }$,
which implies $    \omega_2  \le \sum_{r=0}^{\log(T)} \mathbb{E} \left[ \frac{\mathbb{P} \left\{\theta'_{1,2^r} \le y_j \mid \mathcal{F}_{\lambda^{(1)}_{r}}\right\}}{\mathbb{P} \left\{\theta'_{1,2^r} > y_j \mid \mathcal{F}_{\lambda^{(1)}_{r}} \right\}}  \right]2^{r+1}$.

\textbf{Step 4:} Now, we only need to deal with $\theta'_{1, 2^r} \sim \text{Beta} \left(\widehat{\mu}_{1,2^r} \cdot 2^r +1, (1-\widehat{\mu}_{1,2^r}) \cdot 2^r +1 \right)$, which has nothing to do with differential privacy. 
Thanks to Lemma~2.9 in \citet{agrawal2017near}, we can use their results directly.
\begin{lemma}[Restatement of Lemma~2.9 \citep{agrawal2017near}]
Let $\tau_s^{(1)}$ be the round when the $s{\nth}$ pull of the optimal arm $1$ occurs and let $\theta_{1,s} \sim \text{Beta}\left(\widehat{\mu}_{1,s} \cdot s+1, \left(1- \widehat{\mu}_{1,s}\right) \cdot s +1\right)$, where $\widehat{\mu}_{1,s}$ is the empirical mean of   the optimal arm $1$ among $s$ observations. We have
\begin{equation*}
    \begin{array}{l}
        \mathbb{E} \left[ \frac{\mathbb{P} \left\{\theta_{1,s} \le y_j \mid \mathcal{F}_{\lambda^{(1)}_{s}}\right\}}{\mathbb{P} \left\{\theta_{1,s} > y_j \mid \mathcal{F}_{\lambda^{(1)}_{s}} \right\}}  \right] \le
        
     \begin{cases}
         
         \begin{array}{ll}

 \frac{3}{\mu_1 - y_j}, & \quad s < \frac{8}{\mu_1 - y_j},
 \\
              \Theta \left( e^{-0.5s(\mu_1-y_j)^2}+\frac{e^{-2s(\mu_1-y_j)^2}}{(s+1)(\mu_1-y_j)^2}+\frac{1}{e^{0.25s(\mu_1-y_j)^2}-1}\right), &\quad s \ge \frac{8}{\mu_1-y_j}.  
         \end{array}
        
     \end{cases}   
     \end{array}
\end{equation*}
\end{lemma}
Let $d_1 :=  \log \left(\frac{8}{\mu_1-y_j} \right)$. Then, $\omega_2 \le \sum\limits_{r=0}^{\lfloor d_1 \rfloor} \mathbb{E} \left[ \frac{\mathbb{P} \left\{\theta'_{1,2^r} \le y_j \mid \mathcal{F}_{\lambda^{(1)}_{r}}\right\}}{\mathbb{P} \left\{\theta'_{1,2^r} > y_j \mid \mathcal{F}_{\lambda^{(1)}_{r}} \right\}}  \right]2^{r+1}
     +  \sum\limits_{r= \lceil d_1 \rceil }^{\log (T)} \mathbb{E} \left[ \frac{\mathbb{P} \left\{\theta'_{1,2^r} \le y_j \mid \mathcal{F}_{\lambda^{(1)}_{r}}\right\}}{\mathbb{P} \left\{\theta'_{1,2^r} > y_j \mid \mathcal{F}_{\lambda^{(1)}_{r}} \right\}}  \right]2^{r+1}$,
which can be upper bounded by $O \left(1/\Delta_j^2\right) $ and $O \left(\ln(T)/\Delta_j^2\right) $, separately.
\end{proof-sketch}

\section{Nearly Matching Upper and Lower Bounds Under Full Information} \label{sec:full-info}

In this section, we discuss  differentially private learning algorithms for the full information setting.  We still use the ideas of doubling trick and forgetfulness presented in Section~\ref{sec:bandit} to process the obtained observations. Different from the bandit setting where only the reward of the pulled arm can be observed, under full information feedback, a complete reward vector can be seen regardless of which action is played. Intuitively, we may need more noise to design an $\epsilon$-differentially private learning algorithm under full information feedback.
Since the total privacy budget is $\epsilon$, a naive algorithm, Follow-the-Noisy-Leader (FTNL), is to make  each action's inner algorithm that computes the empirical mean  $(\epsilon/K)$-differentially private. Then, the learning agent takes the action with the highest differentially private empirical mean. However, an $O \left(\max \left\{\frac{\log(K)}{\Delta}, \frac{\log(K)}{\epsilon/K} \right\}  \right)$ regret bound is expected for this naive algorithm, which may not good when the action space is large.

Despite the downside that a complete reward vector can be seen, a nice thing of the full information setting is exploration is not needed. Thus, it is unnecessary to track the detailed differentially private empirical means themselves. Instead, we can simply track the \emph{index} of the action with the highest differentially private empirical mean. 
To get rid of the linear dependency on $K$ for the regret  involving  $\epsilon$, we introduce Report Noisy Max (RNM) of \cite{dwork2014algorithmic} in  FTNL. Algorithm~\ref{Full Information Algorithm} presents the pseudo-code of RNM-FTNL. The algorithm progresses in epochs indexed by $s$ and the action with the highest differentially private aggregated reward is recommended at the end of each epoch. In epoch $s$, the action  recommended at the end of epoch $s-1$ denoted as $J^{(s-1)}$ will be played $2^s$ times. 
Note that epoch $0$ only has one round. For initialization purpose only, we can play any of the actions once.
We initialize $\Phi_j = 0$ to track the aggregated reward of future observations of action $j$ and at the end of epoch $s$, the action $J^{(s)} =  \mathop{\arg\max}_{j \in [K]} \left(\Phi_j + \text{Lap} (1/\epsilon) \right)$ with the highest differentially private aggregated reward is recommended, where $\Phi_j$ denotes the aggregated reward of $2^s$ observations of action $j$. 

 \begin{algorithm}[H]
    \centering
    \begin{algorithmic}[1]
    \caption{RNM-FTNL}
   \label{Full Information Algorithm}
       \footnotesize
     \STATE \textbf{Input:} 
Action set $[K]$ and privacy parameter $\epsilon$ 
\STATE \textbf{Initialization:} Play the first action  once to initialize $J \leftarrow \mathop{\arg\max}_{j \in [K]} \left(X_j + \text{Lap}\left( 1/\epsilon\right) \right)$,  
 $s \leftarrow 1$,  and $\Phi_j \leftarrow 0$ for all $j$ \label{ini-3}
    \WHILE{still have rounds left}

     \STATE Play $J $ for $2^s$ times 

     \STATE Update $J \leftarrow \mathop{\arg\max}_{j \in [K]} \left( \Phi_j + \text{Lap} \left(\frac{1}{\epsilon}\right)\right)$, where $\Phi_j$ is the aggregated reward  among $2^s$ observations
     
     \STATE Set $s \leftarrow s+1$ and reset $\Phi \leftarrow 0$ for all $j \in [K]$.
     \ENDWHILE
     \end{algorithmic}
 \end{algorithm}

 We now present  privacy and regret guarantees for Algorithm~\ref{Full Information Algorithm}.
\begin{theorem}
Algorithm~\ref{Full Information Algorithm} is $\epsilon$-differentially private.
\label{full information DP theorem}
\end{theorem}
\begin{proof-sketch}
Let $X_{1:T}$ be the original reward vector sequence and $X'_{1:T}$ be a neighbouring reward vector sequence that differs in round $t$. Let epoch $s$ be the epoch including round $t$. Clearly, the decisions made in all rounds prior to and including epoch $s$ are the same under either reward vector sequence. Also, from the forgetfulness property of the algorithm, the decisions made in all rounds from  epoch $s+2$  are also the same under either reward sequence. It remains to consider the action played in all rounds in epoch $s+1$, that is, the action recommended by Report Noisy Max (RNM) at the end of epoch $s$. 
We will show the distribution of this action  
is ``almost'' the same whether working over ${X}_{1:T}$ or  ${X}'_{1:T}$. Recall that $J^{(s)}$ indicates the action recommended by RNM at the end of epoch $s$ when working over ${X}_{1:T}$. Let $J'^{(s)} $ indicate the action recommended by RNM when working over ${X}'_{1:T}$. As $\text{Lap} \left( 1/\epsilon \right)$ noise is injected to each action $j$'s aggregated reward, from Claim 3.9 (the RNM algorithm is $\epsilon$-differentially private) of \cite{dwork2014algorithmic}, for any $\sigma \in [K]$, we have 
 $P \left(  J^{(s)} = \sigma \right) 
      \le   e^{\epsilon}  \cdot P\left(  J'^{(s)} = \sigma \right)$,
  which concludes the proof. We defer the full, mathematical proof to Appendix~\ref{app:full-info-dp}.
\end{proof-sketch}

Before presenting our regret upper bound for RNM-FTNL, we first introduce a decomposition of the set of actions into brackets (or levels) of actions whose suboptimality gap is geometrically increasing. To this end, for $\lev = 1, 2, \ldots$, define
\begin{equation}
\begin{array}{l}
\J_\lev := 
  \left\{ 
      j \in {K} \colon 2^{\lev-1} \cdot \Deltamin \leq  \Delta_j < 2^\lev \cdot \Deltamin
  \right\} \quad. 
\end{array}\label{eqn:bracket}
\end{equation}

In addition, let $\Delta_{(\lev)} := 2^{\lev-1} \Deltamin$; hence, $j \in \J_\lev$ implies that $\Delta_{(\lev)} \leq \Delta_j< 2 \Delta_{(\lev)}$.

\begin{theorem} \label{thm:ftnl-regret}
The instance-dependent reget of RNM-FTNL is
\begin{align*}
O \left( 
  \frac{\log (K)}{\Deltamin} 
  + \sum_{\lev \geq 1 \colon |\J_\lev| > 0, \Delta_{(\lev)} > \epsilon} \frac{1 + \log |\J_\lev|}{\epsilon}
\right) 
= O \left(
    \frac{\log(K)}{\Deltamin} 
    + \frac{\log(K) \log \left( \frac{\Delta_{\max}}{\max\left\{ \Deltamin, \epsilon, \frac{1}{T} \right\}} \right)}{\epsilon}
    \right) .
\end{align*}
\end{theorem}

Before sketching a proof of Theorem~\ref{thm:ftnl-regret}, let us explore the implications of this result. 
The following corollary illustrates the power of the theorem with some examples.
\begin{corollary} \label{cor:ftnl-regret}
The following are true: (a) If all $\Delta_j \leq \epsilon$,  RNM-FTNL has regret 
at most $O \left( \log(K)/\Deltamin \right)$. (b)
If all $\Delta_j \in [\Deltamin, 2 \Deltamin]$,  RNM-FTNL has regret 
at most $O \left( \frac{\log(K)}{\min\{ \Deltamin, \epsilon\} } \right)$.
\end{corollary}
\begin{proof}
For \emph{(a)}, note that if $j \in \J_\lev$, then $\Delta_{(\lev)} \leq \Delta_j \leq \epsilon$. Therefore, the second term in the first bound of Theorem~\ref{thm:ftnl-regret} vanishes. The bound \emph{(b)} is immediate from the second bound of Theorem~\ref{thm:ftnl-regret} since $\frac{\Delta_{\max}}{\Deltamin} = O(1)$.\end{proof}

By modifying the proof of Theorem~\ref{thm:ftnl-regret} slightly, we have  the following instance-independent regret bound.
\begin{corollary}[Instance-independent bound] \label{cor:ftnl-indep-regret}
For any problem instance, RNM-FTNL's regret is at most
$O \left( \sqrt{T \log K} + \log(K) \log(T)/\epsilon \right)$.
\end{corollary}
\begin{proof}
We use a well-known idea that appeared at least as early as a work of \citet[discussion after Theorem 2]{audibert2009minimax}. 
Let $ \Delta^* := \sqrt{\log K/T}$ be the critical gap.  
For all epochs in which the learning algorithm takes actions $j$ satisfying $\Delta_j < \Delta^*$, the total contribution to the regret can be at most $T\Delta^* = \sqrt{T \log K}$. 
Next, we bound the regret contribution from actions $j$ for which $\Delta_j \geq \Delta^*$. To do so, we may simply adapt the proof of Theorem~\ref{thm:ftnl-regret} to consider only these suboptimal actions, so that the effective $\Deltamin$ becomes $\Delta^*$. The second bound of Theorem~\ref{thm:ftnl-regret} then bounds the regret contribution as $O \left(
\frac{\log(K)}{\Delta^{*}} 
+ \frac{\log(K) \log \left( \Delta_{\max} / \max\left\{ \Delta^{*}, \epsilon, 1/T \right\} \right)}{\epsilon}
\right) = O \left( \sqrt{T \log (K)} + \frac{\log(K) \log(T)}{\epsilon} \right)$.
\end{proof}

A previous version of this work \citep[arXiv:2102.07929v1]{hu2021optimal} contained exactly the same algorithm as we present here. In that version, we claimed that the regret is
$O \left( \frac{\log (K)}{\min\{\Deltamin, \epsilon\}} \right)$.
However, there was a mistake in one part of the analysis. In more detail, in equation (54) of \citep[arXiv:2102.07929v1]{hu2021optimal}, the first inequality does not hold in general. For example, when $\epsilon$ is small enough, we can have that $d_1 
= d_2 - 1 
= d_3 - 2 
= \dots 
= d_{r_{\max}} + 1 - r_{\max} $,
which implies that we may have to pay an extra factor of order
$\min \left\{ 
     \log \left( 1/\Deltamin \right), \log(T)
\right\}$. We thus far have not succeeded in proving the originally claimed
$O \left( \log (K)/\min\{\Deltamin, \epsilon\} \right)$
regret bound. 

More recently, we discovered a later work \citep{asi2022private} that presented a instance-independent regret bound for the pseudo-regret\footnote{While their results are stated for the expected regret, from their proofs it appears that they actually bound the pseudo-regret (which we remind the reader is abbreviated to regret throughout our work).} (see Corollary 3.4 of \citet{asi2022private}). Their approach is to reduce the online setting to the offline differentially private stochastic convex optimization setting, and their instance-independent regret bound is of order $\sqrt{T \log K} + \log(K) \log(T)/\epsilon$, which is the same rate that we obtain in our Corollary~\ref{cor:ftnl-indep-regret}. Similar to our approach, they also use an epoch-based algorithm which, in the current epoch, passes the samples from the previous epoch to a differentially private offline algorithm and then plays the resulting hypothesis (action) for all rounds of the current epoch. 
The differentially private offline algorithm that we adopt within each epoch is simple: it can be viewed as differentially private ERM. 
Their offline\footnote{In fact, it is an online algorithm, but we call it offline as it only produces a single action at the end.} algorithm, Private Variance Reduced Frank-Wolfe, is more sophisticated, but we stress that they solve a more general problem. It is not clear to us if their approach also can be used to obtain instance-dependent bounds. Even if it could be used to obtain such bounds (perhaps by first showing instance-dependent bounds for Private Variance Reduced Frank-Wolfe), their style of analysis adds up the regret bound for each epoch separately. Although our analysis of course also proceeds in epochs, we tune certain parameters in our regret analysis in a global way, i.e., taking into consideration all the epochs. We believe this is what allows our regret bound to, in certain special cases, have a privacy price of only $\log (K)/\epsilon$ rather than $\log(K) \log(T)/\epsilon$.
Although both our works have a privacy price of $\log(K) \log(T)/\epsilon$, we conjecture that the optimal rate should have a price of either $\log(K)/\epsilon$ or $\log(T)/\epsilon$. We have not yet succeeded in obtaining either of these rates. 

We now sketch a proof of Theorem~\ref{thm:ftnl-regret}. The full proof can be found in Appendix~\ref{app:full-info-upper-bound}.

\

\begin{proof-sketch}(of Theorem~\ref{thm:ftnl-regret})
The idea of the proof is to first decompose the set of actions into brackets of geometrically increasing gaps, as in \eqref{eqn:bracket}. We then bound the contribution to the regret from each bracket separately by bounding the expected number of plays of an action in each bracket. Fix a bracket $\J_\lev$ indexed by some integer $\lev$. For this bracket, we bound the regret by partitioning the rounds into two contiguous blocks: the \emph{undersampled regime} consists of round 1 up to some ``sufficiently sampled'' threshold, and the \emph{sufficiently sampled} regime consists of all the remaining rounds. Our actual analysis proceeds by epochs, but this only requires small, less interesting changes to the analysis. In the undersampled regime, we simply assume that a suboptimal action in bracket $\J_\lev$ is played in each round. In the sufficiently sampled regime, we use Hoeffding-style (sub-Gaussian) and Laplace concentration to control the sum, over rounds (actually epochs), of the probability of playing an  action in bracket $\J_\lev$ in each round. Finally, we tune the sufficiently sampled threshold to minimize the expected total number of plays of an action in this bracket. By executing the above steps, we show that the regret contribution from bracket $\J_\lev$ is at most
$O \left(
    \frac{1 + \ln |\J_\lev|}{\min\{\Delta_{(\lev)}, \epsilon\}}
  \right)$.
Summing over all brackets and some basic manipulations gives the first bound in the theorem. For the second bound, we start with the first bound in the theorem and apply the trivial upper bound $|\J_\lev| \leq K$ for all $\lev$. 
Next, we bound the number brackets in the summation in the first bound by the logarithm (due to the geometrically increasing property) of the range of $\Delta_{(\lev)}$. The largest value of $\Delta_{(\lev)}$ is $O(\Delta_{\max})$. For the smallest value, clearly we can take the maximum of $\Deltamin$ and $\epsilon$; moreover, without loss of generality, we may ignore brackets for which the gaps are $O(1/T)$ as these brackets can contribute at most $O(1)$ to the regret.
\end{proof-sketch}

To complement the above upper bounds, we also show an instance-dependent regret lower bound and then briefly explain how this also implies a minimax regret lower bound.

\begin{theorem} \label{thm:dp-dtol-lower-bound}
Let $\Deltamin \in (0, 1/4)$, 
$K \geq 4$, and 
$T \geq \max \left\{ \frac{\ln (K)}{16 \Deltamin^2}, \frac{\ln (K)}{10 \epsilon \Deltamin} \right\}$. 
Then, for any $\epsilon$-differentially private online learning algorithm for the 
full information setting, 
there is a distribution over reward vectors (with minimum gap $\Deltamin$ as above) such that the optimal action is unique and the regret of the algorithm is lower bounded as 
$\Omega \left( \max \left\{ 
     \frac{\log (K)}{\Deltamin},
     \frac{\log (K)}{\epsilon}
     \right\} \right)$. 
\end{theorem}

We will sketch a proof of this result momentarily; the full proof is left to Appendix~\ref{app:lower-bound}. 
The first term in the lower bound is from the non-private setting (see Proposition 4 of \citet{mourtada2019optimality}) and requires $\Deltamin > 1/4$, $K \geq 4$ and 
$T \geq \frac{\ln K}{16 \Deltamin^2}$.
The second term is the new one, coming from the restriction of $\epsilon$-differential privacy, and establishing it requires only $T \geq \frac{\ln K}{10 \epsilon \Deltamin}$. Note that by taking $\Deltamin = \sqrt{\ln K/T}$ (as mentioned by \citet{mourtada2019optimality} just after their Proposition 4), we recover the minimax (i.e., instance-independent) lower bound --- now also with a privacy-related term --- of order $\max\left\{ \sqrt{T \log K}, \log K/\epsilon \right\}$, which holds provided that $T \geq \frac{\ln K}{100 \epsilon^2}$. The lower bound of Theorem~\ref{thm:dp-dtol-lower-bound}  does not match the best known upper bound, which, as mentioned earlier, is of order $\log K/\Deltamin + \log(K) \log(T)/\epsilon$. Closing the gap between the upper and lower bounds is a very exciting direction.

\begin{proof-sketch}(of Theorem~\ref{thm:dp-dtol-lower-bound})
The non-private  $\Omega \left(\log(K)/\Deltamin  \right)$ term is from Proposition 4 of \citet{mourtada2019optimality}. We now sketch a proof for the $\Omega \left( \log(K)/\epsilon \right)$ term. 
The proof is in three steps.

\textbf{Step 1:} First, we construct a set of $K$ problem instances. Each problem instance is specified by a product distribution composed of $K$ reward distributions, with each action associated with a reward distribution. The two key properties of the construction are that:
\begin{enumerate}
\item in each problem instance, a different action has the highest mean reward;
\item for any two problem instances, the total variation distance between their corresponding product distributions is of order $\Deltamin$. 
\end{enumerate}
Note that the second property means the product distributions form a packing.

\textbf{Step 2:} Next, we consider the multiple hypothesis testing problem corresponding to the packing from Step 1. For any $\epsilon$-differentially private algorithm, Corollary 4 of \citet{acharya2021differentially} gives an $\Omega\left( \log K/\left(\Deltamin \cdot \epsilon \right) \right)$ lower bound on the sample complexity of this hypothesis testing problem (for a constant failure probability).

\textbf{Step 3:} The final step is a ``batch-to-online'' conversion which, given a sample complexity lower bound for the multiple hypothesis testing problem, yields a lower bound on the regret of any $\epsilon$-differentially private online learning algorithm for the full information online learning problem. 
The key intuition is that any such online learning algorithm can be transformed into a hypothesis selector by taking the majority vote among the hypotheses that were played in each round. Therefore, if the learning algorithm is run for a number of rounds strictly less than the sample complexity $m$, then the best action cannot be the majority vote (if it were, then the algorithm is successful on the hypothesis testing problem). Consequently, in each of at least roughly $0.5 \cdot m$ rounds, a suboptimal action must have been played, giving regret $\Omega\left(\log (K)/\epsilon \right)$.\end{proof-sketch}

\section{Experimental Results}  \label{sec: empirical results}
We conduct numerical experiments in two settings to compare the empirical performance of Anytime-Lazy-UCB, Lazy-DP-TS, DP-SE of \cite{sajed2019optimal}, and AdaP-KLUCB of \cite{asi2022private}. In both settings, we set the number of rounds $T=10^6$  and the number of arms $K=5$. The privacy parameter is set to $\epsilon=\left\{0.25, 0.5,1\right\}$.
In the first setting, all the mean rewards are set to $\left\{0.75, 0.625, 0.5, 0.375, 0.25 \right\}$ with different mean reward gaps.
For the second setting, all the mean rewards are set to $\left\{0.5, 0.4, 0.4, 0.4, 0.4 \right\}$ with the same mean reward gap.
All the reported results in  Figures~\ref{fig:resultsetting1} and~\ref{fig:resultsetting2} are the average of $20$ independent runs. 
Not surprisingly, the empirical performance of Lazy-DP-TS is similar to AdaP-KLUCB due to the facts that:
 \begin{enumerate}
     \item the non-private version of Lazy-DP-TS (Thompson Sampling with Beta Priors \citep{agrawal2017near}) and the non-private version of AdaP-KLUCB (KL-UCB \citep{garivier2011kl}) have the same coefficient for the leading $\log(T)$ term in the instance-dependent regret bound;
     \item both AdaP-KLUCB and Lazy-DP-TS use, essentially, the same way to add noise. Therefore, it is also expected that the coefficient for the $\log(T)/\epsilon$ privacy term is also the same for these two algorithms.
 \end{enumerate}
Additionally, as compared with DP-SE and Anytime-Lazy-UCB,  Lazy-DP-TS always empirically performs better.

\begin{figure}[h]
\centering
\subfloat [$\epsilon = 0.25$] {\label{fig:result1ep025}
\includegraphics[width = 0.33\columnwidth]{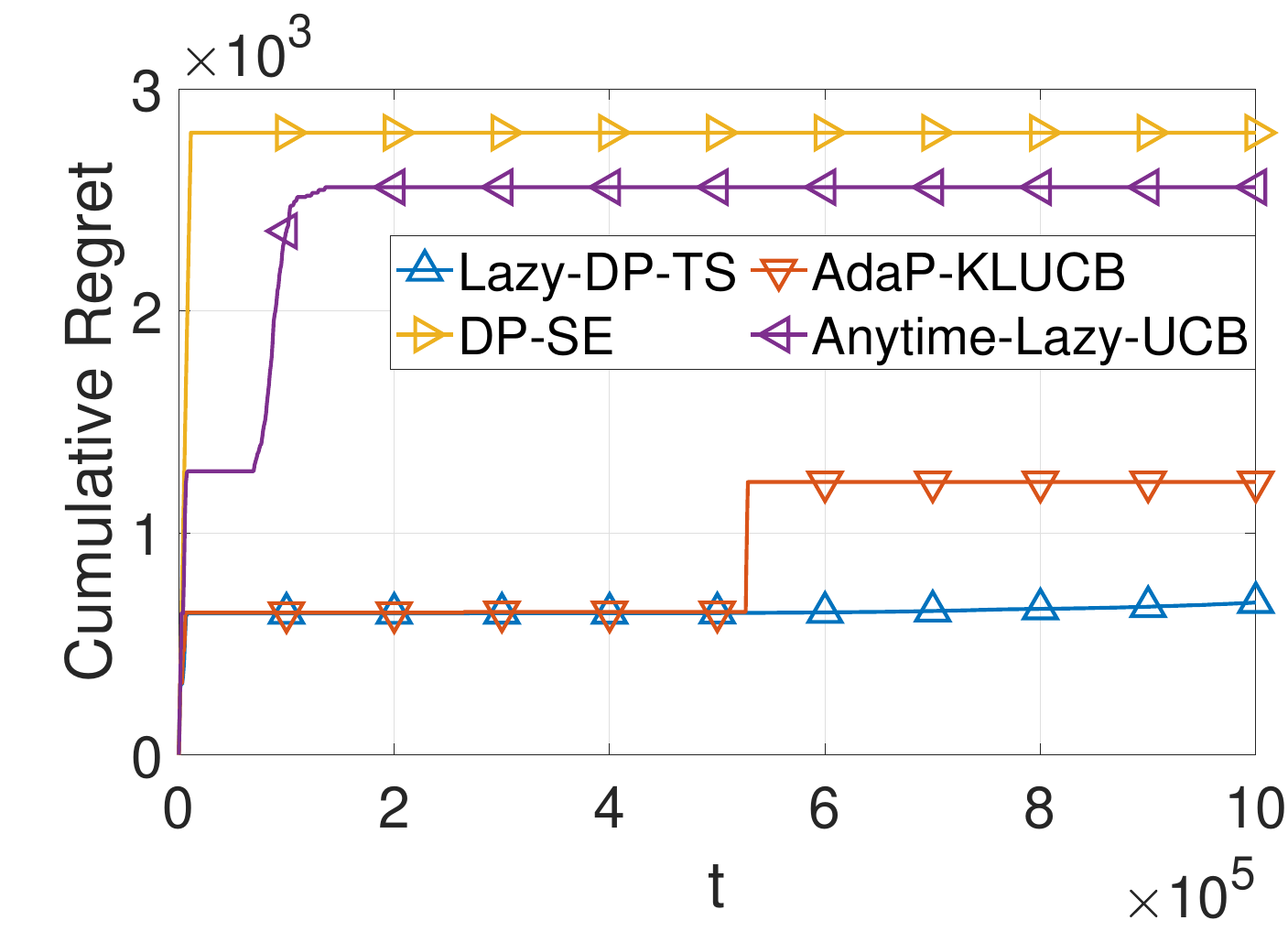}
}
\subfloat [$\epsilon = 0.5$] {\label{fig:result1ep05}
\includegraphics[width = 0.33\columnwidth]{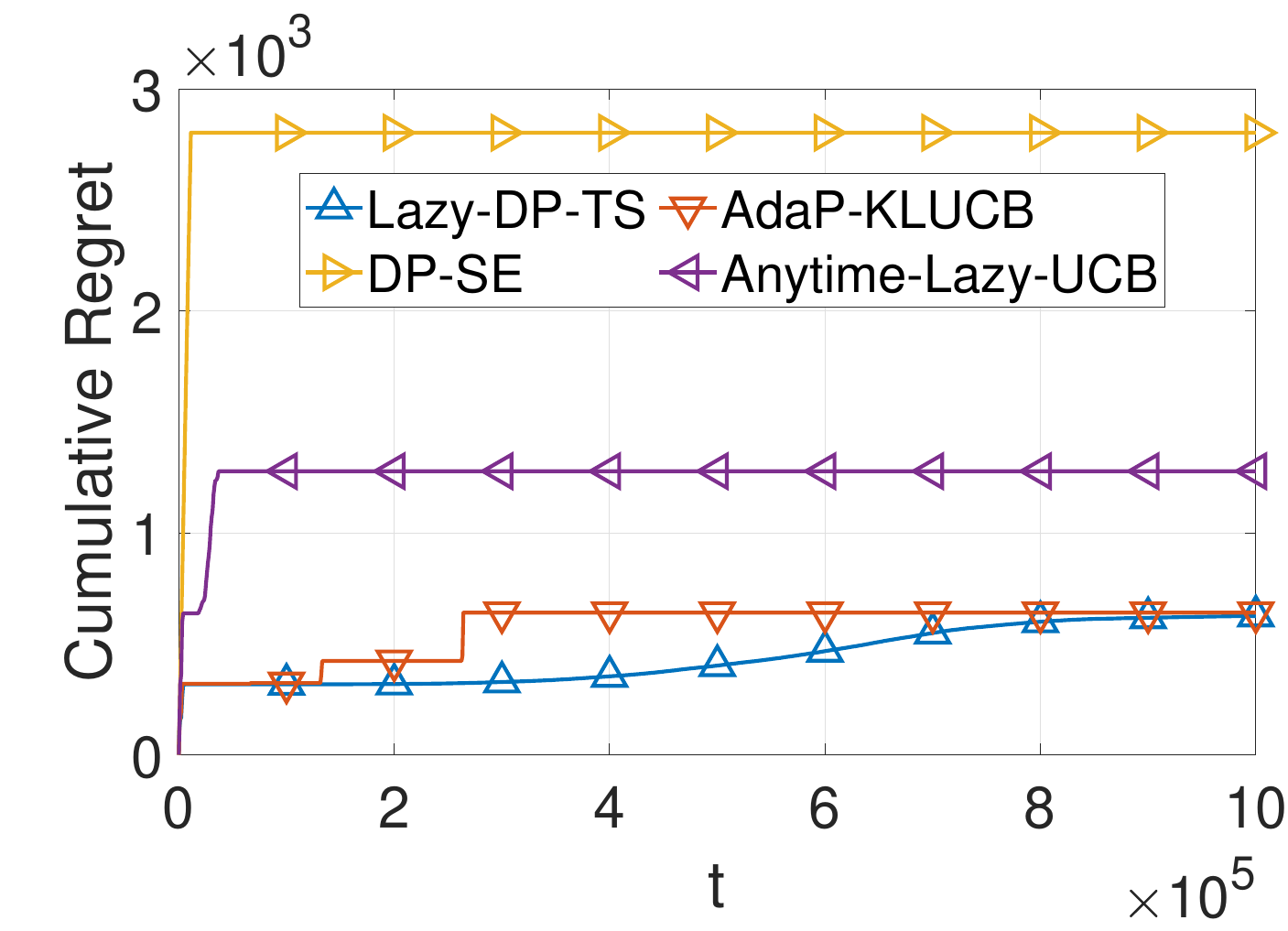}
}
\subfloat [$\epsilon = 1$] {\label{fig:result1ep1}
\includegraphics[width = 0.33\columnwidth]{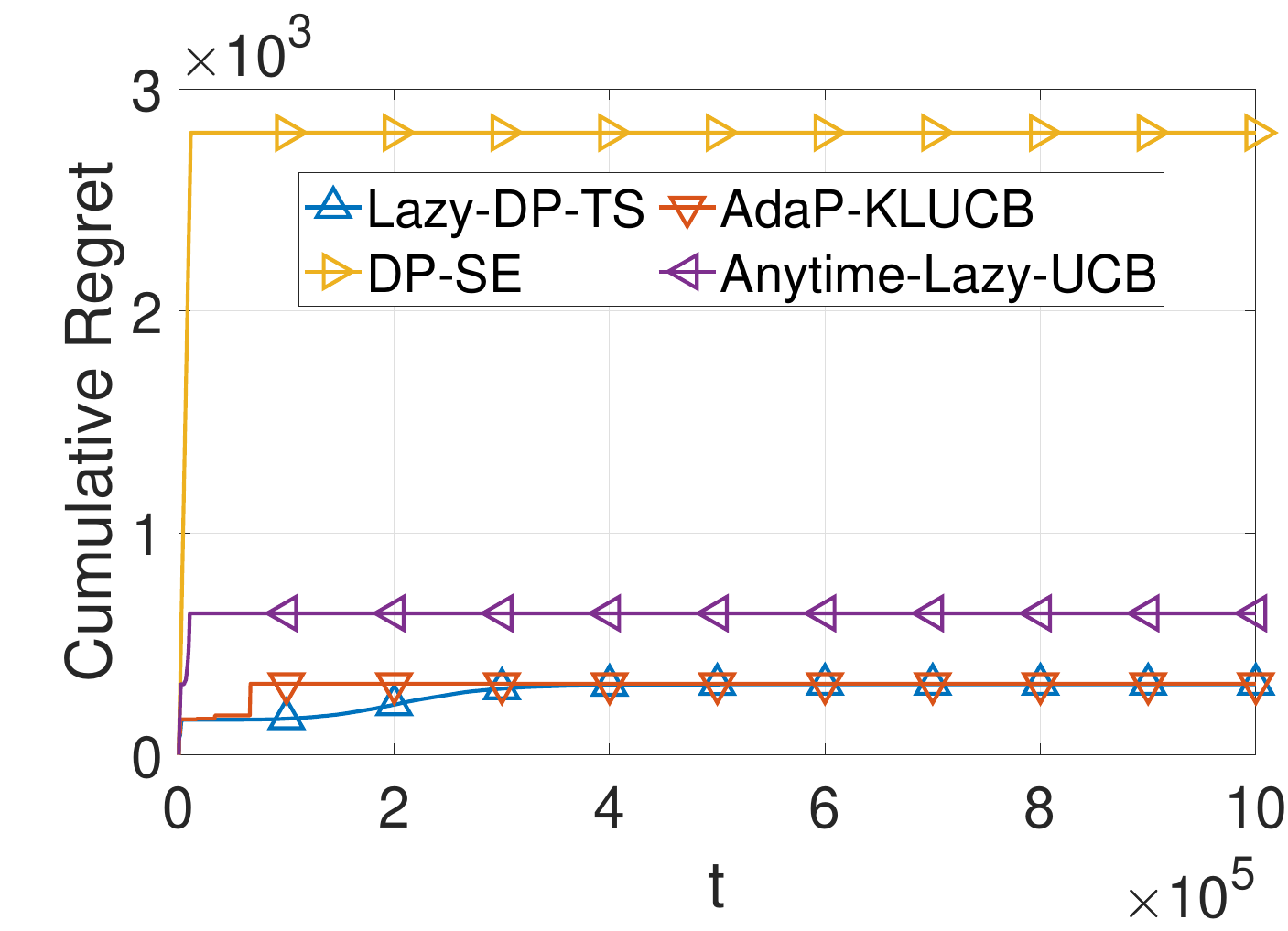}
}
\caption{The cumulative regret for the first setting.}\label{fig:resultsetting1}
\end{figure}

\begin{figure}[h]
\centering
\subfloat [$\epsilon = 0.25$] {\label{fig:result2ep025}
\includegraphics[width = 0.33\columnwidth]{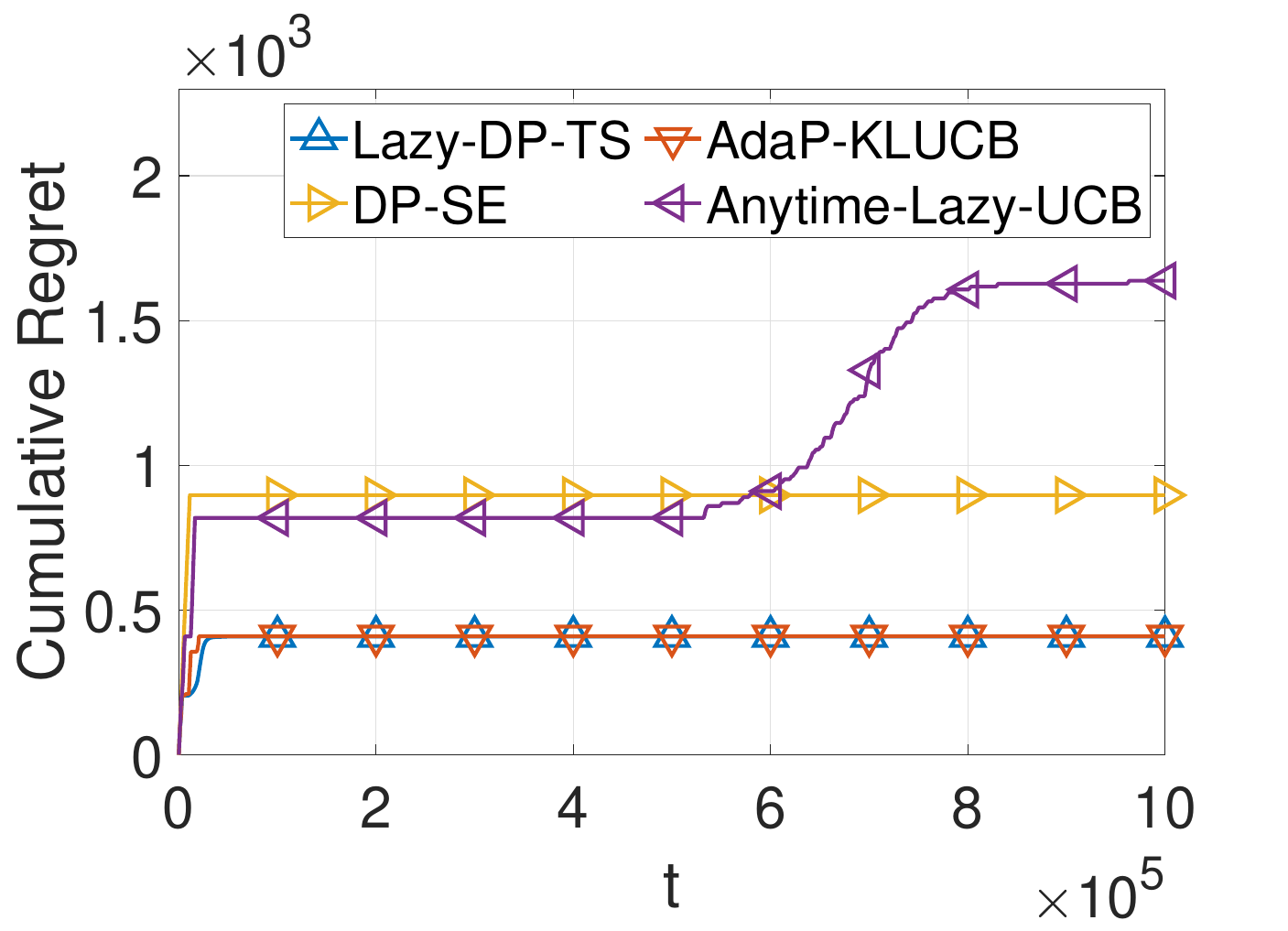}
}
\subfloat [$\epsilon = 0.5$] {\label{fig:result2ep05}
\includegraphics[width = 0.33\columnwidth]{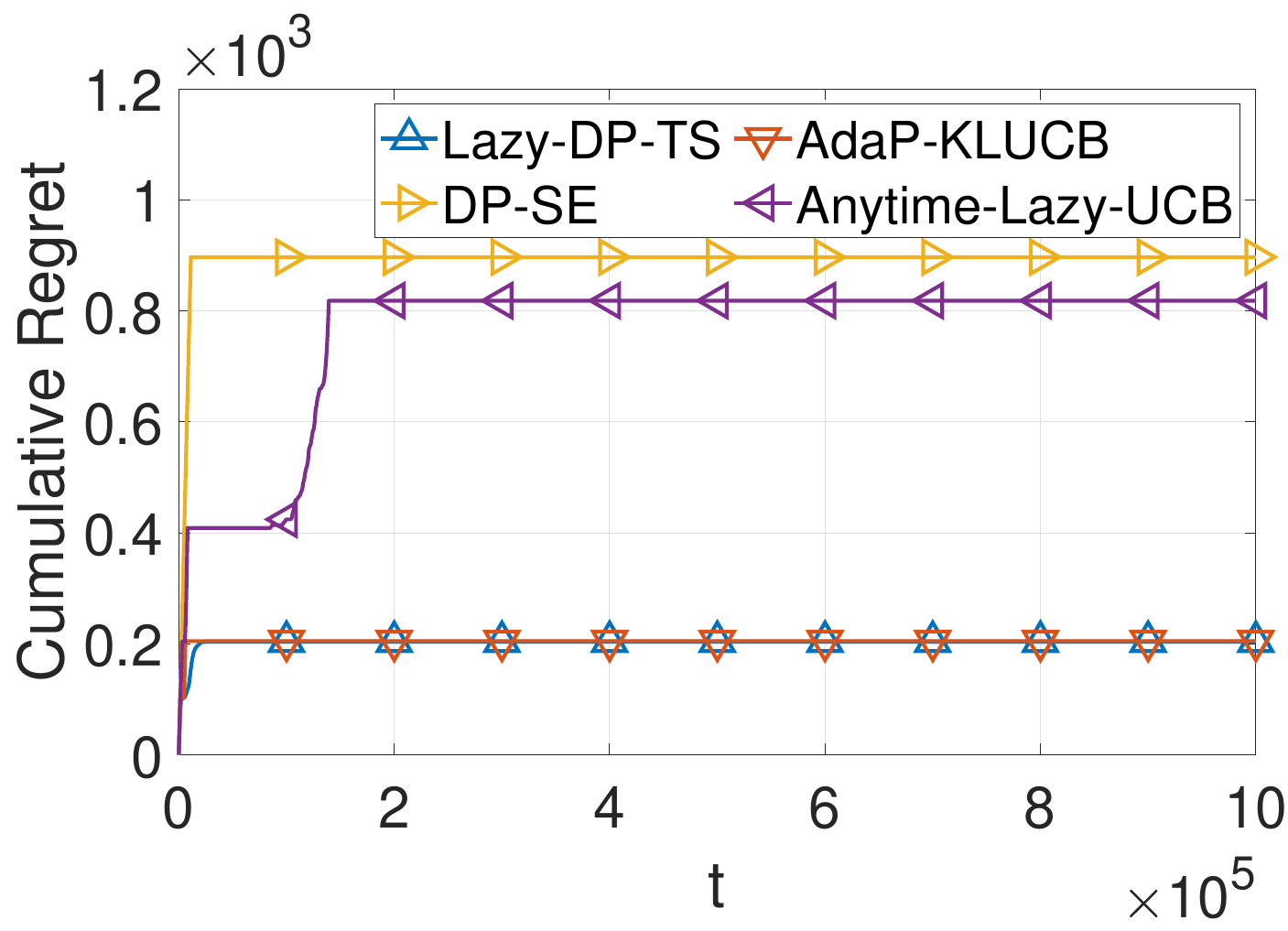}
}
\subfloat [$\epsilon = 1$] {\label{fig:result2ep1}
\includegraphics[width = 0.33\columnwidth]{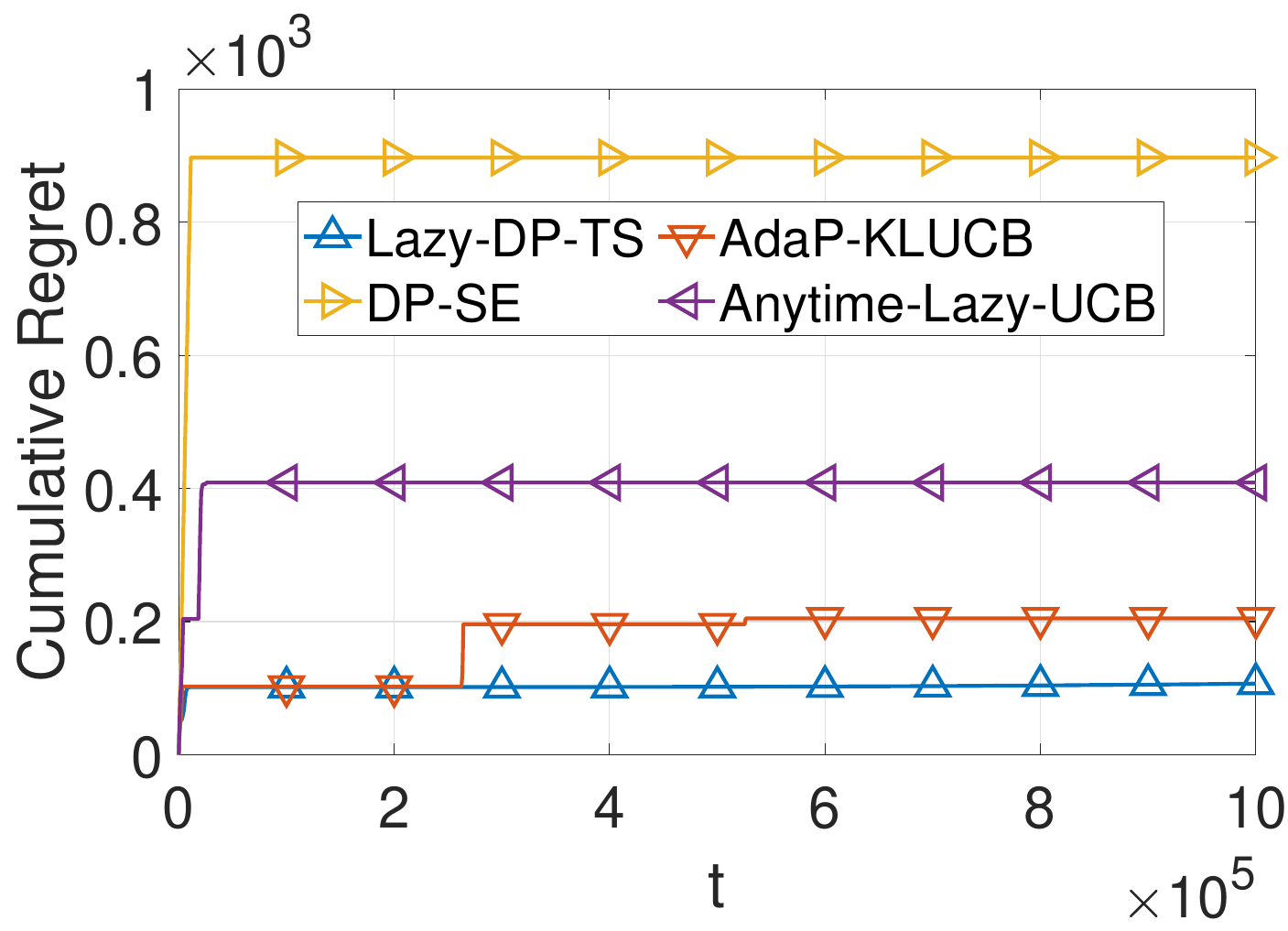}
}
\caption{The cumulative regret for the second setting.}\label{fig:resultsetting2}
\end{figure}

\section{Conclusion} \label{sec:conclusion}

In the differentially private stochastic bandit setting,  we have presented both the optimal UCB-based algorithm (Anytime-Lazy-UCB) and Thompson Sampling-based  algorithm (Lazy-DP-TS), which fills a gap in the differentially private bandit literature. We believe that the ideas of maintaining an arm-specific epoch with doubling length and updating the differentially private empirical mean in a lazy and forgetful way  contribute to developing optimal algorithms for other differentially private  online learning variants beyond bandits.
Although Thompson Sampling-based algorithms instrincally randomize the collected data, both of our proposed algorithms use the same framework to preserve privacy.
One of the interesting future directions is to devise private Thompson Sampling-based algorithms  leveraging the fact that posterior samples  provide some degree of differential privacy for free.

In the differentially private full information setting, there is a gap between our instance-dependent upper and lower bounds. Surprisingly, for fixed $\Deltamin$, the ``hard problem instance'' used to prove the instance-dependent lower bound (Theorem~\ref{thm:dp-dtol-lower-bound}) is \emph{not} the instance which maximizes our upper bound (Theorem~\ref{thm:ftnl-regret}). Indeed, in the ``hard problem instance'' all suboptimal actions $j$ have $\Delta_j = \Deltamin$; our upper bound then achieves the rate $\frac{\log (K)}{\min\{\Delta, \epsilon\}}$, the same rate given by our lower bound. Rather, the witness of the largest discrepancy between our upper and lower bounds is when we have the ``doubling gap instance'' $\Delta_j = \Deltamin \cdot 2^{j-2}$ for $j = 2, 3, \ldots, K$ for $\epsilon \leq \Deltamin$, in which case our upper bound is of order
$\frac{\log (K)}{\Deltamin} + \frac{\log(K) \log(T)}{\epsilon}$.
Intuitively, the ``doubling gap instance'' should not be more difficult than the instance used to prove the lower bound. Therefore, we conjecture that the rate of $\frac{\log (K)}{\min\{\Deltamin, \varepsilon\}}$ is optimal. It would be very interesting to resolve the gap between the upper and lower bounds. 

\vskip 0.2in
\bibliography{arxiv_v3_dp_stoch_online_learning}

\newpage

\appendix
\begin{center}
   \Large \textbf{APPENDIX}
\end{center}
The appendix is organized as follows.
\begin{enumerate}

\item Useful facts 
in Appendix~\ref{app: fact};
    \item Proofs for Anytime-Lazy-UCB in Appendix~\ref{app: Anytime-Lazy-UCB};
       \item Proofs for Lazy-DP-TS in Appendix~\ref{app: Lazy-DP-TS};
       
       \item Proofs for RNM-FTNL in Appendix~\ref{app:ftnl}.

    \end{enumerate}

\section{Useful facts} \label{app: fact}
We first present some useful facts.
\begin{fact}[Beta distribution] The probability density function of Beta distribution $\text{Beta}(\alpha, \beta)$ with parameters $\alpha, \beta >0$ is
\begin{equation}
    f(x; \alpha, \beta) = \frac{x^{\alpha-1} (1-x)^{\beta-1}}{\int_{0}^{1}x^{\alpha-1} (1-x)^{\beta-1} dx} \quad.
\end{equation}
\label{pdf of beta}
\end{fact}

\begin{fact}[Laplace distribution] The probability density function of Laplace distribution $\text{Lap}(b)$ centered at $0$ with scale
$b$ is 
\begin{equation}
    f(x; b) = \frac{1}{2b} e^{- \frac{|x|}{b}} \quad.
\end{equation}
\label{pdf of lap}
\end{fact}

\begin{fact}[Hoeffding's inequality] 
Let $X_1, X_2, \dotsc, X_n$  be $n$ independent  random
variables with common range $[0,1]$. Let $X = \sum_{j=1}^{n}X_j$. Then, for any $a > 0$, we have
\begin{equation}
   \mathbb{P} \left\{\left|X -\mathbb{E}[X] \right| \ge a \right\} \le 2e^{-2a^2/n} \quad.
\end{equation}
\label{Hoeffding}
\end{fact}
\begin{fact}[Concentration bound for $\text{Lap}(b)$, Fact 3.7 in \cite{dwork2014algorithmic}] 
If $Y \sim \text{Lap}(b)$, for any $0 < \delta < 1$, we have
\begin{equation}
   \mathbb{P} \left\{ |Y| > \ln \left( \frac{1}{\delta}  \right)\cdot b \right\} = \delta \quad.
\end{equation}
\label{fact 1}
\end{fact}

\begin{fact}[Chernoff-Hoeffding bound] 
Let $X_1, X_2, \dotsc, X_n$  be $n$ independent random
variables with each $X_j \in \{0,1\}$. Let $X = \frac{1}{n} \sum_{j=1}^{n}X_j$ and $\mu = \mathbb{E} \left[ X \right]$. Then, for any $0 < \lambda < 1- \mu$, we have
\begin{equation}
    \mathbb{P} \left\{X \ge \mu+\lambda \right\} \le e^{-n \cdot d_{\text{KL}} \left(\mu+\lambda, \mu\right)}\quad,
\end{equation}
and for any $0 < \lambda < \mu$, we have
\begin{equation}
    \mathbb{P} \left\{X \ge \mu- \lambda \right\} \le e^{-n \cdot d_{\text{KL}} \left(\mu-\lambda, \mu\right)}\quad,
\end{equation}
where $d_{\text{KL}}(x,y) = x\ln \left(\frac{x}{y}\right)+(1-x) \ln \left(\frac{1-x}{1-y}\right)$ denotes the KL-divergence between two Bernoulli distributions with parameters $x$ and $y$. 
\label{chernoff-hoeffding}
\end{fact}

\begin{fact}
For all positive integers $\alpha, \beta$, we have
\begin{equation}
    F_{\alpha, \beta}^{\text{Beta}}(y) = 1-F_{\alpha+ \beta-1,y}^{\text{B}}(\alpha-1) \quad,
\end{equation}
where $F_{\alpha, \beta}^{\text{Beta}}(\cdot)$ denotes the cumulative density function of Beta distribution with parameters $\alpha, \beta$ and  $F_{n,p}^{\text{B}}(\cdot)$ denotes the cumulative density function of  Binomial distribution with parameters $n,p$. \label{cdfs}
\end{fact}

\section{Proofs for Anytime-Lazy-UCB} \label{app: Anytime-Lazy-UCB}

The full proof for Theorem~\ref{DP: Optimal DP-UCB} is shown in Section~\ref{app: dp}, and the full proof for Theorem~\ref{Regret: Optimal DP-UCB} is shown in Section~\ref{app: Lazy-DP-UCB}.
\subsection{Proofs for Theorem~\ref{DP: Optimal DP-UCB}} \label{app: dp}
\begin{proof}(of Theorem~\ref{DP: Optimal DP-UCB})
 Let ${X}_{1 : T}$ be the original reward vector sequence and ${X}'_{1 : T}$ be an arbitrary neighbouring reward vector sequence of ${X}_{1 : T}$ such that they can differ in an arbitrary round. Let us say ${X}_{1 : T}$ and ${X}'_{1 : T}$ differ from each other in round $t$.
  Let $D_{1:T}$ be the sequence of pulled arms from round $1$ to round $T$ when working over ${X}_{1 : T}$. Let $D_{1:T}'$ be the sequence of pulled arms  when working over ${X}'_{1 : T}$.

For an arbitrary point  $\sigma_{1 : T} \in [K]^T$, we claim the probability mass functions of $D_{1:T}$ and $D'_{1:T}$ satisfy
\begin{equation}
      P\left(  D_{1:T}= \sigma_{1 : T} \mid X_{1:T}\right) \le e^{\epsilon} \cdot P  \left( D_{1:T}'= \sigma_{1 : T} \mid X'_{1:T} \right) \quad.
 \end{equation}
Let $P \left( \cdot \mid X_{1:T} \right) := P_{X} \left( \cdot\right)$ and $P \left( \cdot \mid X'_{1:T} \right) := P_{X'} \left( \cdot\right)$. To prove the above claim, it is sufficient to show
\begin{equation}
      P_{X}\left(  D_{t+1:T}= \sigma_{t+1 : T} \mid D_{1:t}= \sigma_{1 : t}\right) \le e^{\epsilon} \cdot P_{X'}  \left(D_{t+1:T}'= \sigma_{t+1 : T}  \mid D_{1:t}'= \sigma_{1 : t} \right) 
 \end{equation}
 due to  $  P\left(  D_{1:t}= \sigma_{1 : t} \mid X_{1:T}\right) = P  \left(D_{1:t}'= \sigma_{1 : t}  \mid X'_{1:T}\right)$. Note that  $X_{1:t-1}=X'_{1:t-1}$.

Since the learning algorithm processes the obtained observations in a delayed manner and based on epochs, we let  
 $t_* \ge t$  be the  round such that, at the end of this round, the obtained observation   in round $t$ will be processed.
We have
\begin{equation}
\begin{array}{ll}
 & P_X\left(  D_{t+1:T}= \sigma_{t+1 : T} \mid D_{1:t}= \sigma_{1 : t}\right) \\
 &\\
= &\sum\limits_{s \ge t}  \underbrace {P_{X} \left(  D_{t+1:t_*}= \sigma_{t+1 : t_*}  \mid D_{1:t}= \sigma_{1 : t},  t_*=s \right) \cdot  P_{X} \left(   t_* = s \mid D_{1:t}= \sigma_{1 : t} \right) }_{\beta}\\
   & \underbrace{P_{X} \left(  D_{t_*+1:T}= \sigma_{t_*+1 : T}  \mid D_{1:t_*}= \sigma_{1 : t_*},  t_*=s \right)}_{\alpha}
\end{array}
\end{equation}
and
\begin{equation}
\begin{array}{ll}
  & P_{X'}\left(  D'_{t+1:T}= \sigma_{t+1 : T} \mid D'_{1:t}= \sigma_{1 : t}\right) \\
  & \\
= &\sum\limits_{s \ge t}  \underbrace{ P_{X'} \left(  D'_{t+1:t_*}= \sigma_{t+1 : t_*}  \mid D'_{1:t}= \sigma_{1 : t},  t_*=s \right) \cdot  P_{X'} \left(   t_* = s \mid D'_{1:t}= \sigma_{1 : t} \right)}_{\beta'} \\
   & \underbrace{P_{X'} \left(  D'_{t_*+1:T}= \sigma_{t_*+1 : T}  \mid D'_{1:t_*}= \sigma_{1 : t_*},  t_*=s \right)}_{\alpha'}\quad.
\end{array}
\end{equation}
It is not hard to see we have $\beta = \beta'$ as the obtained observation in round $t$ will only be used from round $t_*+1$. 
Now, we show $\alpha \le e^{\epsilon} \cdot \alpha' $ to complete the proof. 
Let $\widetilde{A}_{\sigma_t} $ denote the output of the Laplace mechanism for the function that computes  the aggregated  reward of arm $\sigma_t$. From Theorem~3.6 in \cite{dwork2014algorithmic} that Laplace mechanism preserves $\epsilon$-differential privacy, for any $y \in \mathbb{R}$, we have the conditional density functions satisfy
\begin{equation}
  \begin{array}{lll}
f_{\widetilde{A}_{\sigma_t}}(y \mid D_{1:t_*} = \sigma_{1:t_*}, t_* = s, X_{1:T}) & \le &e^{\epsilon} \cdot f_{\widetilde{A}_{\sigma_t}}(y \mid D_{1:t_*} = \sigma_{1:t_*}, t_* = s, X'_{1:T})\\
& = & e^{\epsilon} \cdot f_{\widetilde{A}_{\sigma_t}}(y \mid D'_{1:t_*} = \sigma_{1:t_*}, t_* = s, X'_{1:T})\quad.
    \end{array}
\end{equation}
We rewrite $\alpha$ and have
\begin{equation}
\begin{array}{ll}
& P_{X} \left(  D_{t_*+1:T}= \sigma_{t_*+1 : T}  \mid D_{1:t_*}= \sigma_{1 : t_*} ,  t_*=s \right) \\
 = & \int_{-\infty}^{+\infty} P_{X} \left(  D_{t_*+1:T}= \sigma_{t_*+1 : T}  \mid D_{1:t_*}= \sigma_{1 : t_*} ,  t_*=s, \widetilde{A}_{\sigma_t} = y\right) f_{\widetilde{A}_{\sigma_t} }(y  \mid D_{1:t_*}= \sigma_{1 : t_*} ,  t_*=s, X_{1:T}) dy\\
  \le & \int_{-\infty}^{+\infty} P_{X'} \left(  D'_{t_*+1:T}= \sigma_{t_*+1 : T}  \mid D'_{1:t_*}= \sigma_{1 : t_*} ,  t_*=s, \widetilde{A}_{\sigma_t} = y\right) \cdot e^{\epsilon} \cdot  f_{\widetilde{A}_{\sigma_t} }(y  \mid D'_{1:t_*}= \sigma_{1 : t_*} ,  t_*=s, X'_{1:T}) dy\\
    = & e^{\epsilon} \cdot P_{X'} \left(  D'_{t_*+1:T}= \sigma_{t_*+1 : T}  \mid D'_{1:t_*}= \sigma_{1 : t_*} ,  t_*=s \right) \\
    = & e^{\epsilon} \alpha'\quad,
\end{array}
\end{equation}
which concludes the proof.\end{proof}

\subsection{Proofs for Theorem~\ref{Regret: Optimal DP-UCB}} \label{app: Lazy-DP-UCB}
The instance-dependent regret bound proof is shown in Section~\ref{app: isolated 1}, and the instance-independent regret bound proof is shown in Section~\ref{app: isolated 2}.

\subsubsection{Proofs for the instance-dependent regret bound} \label{app: isolated 1}

As already mentioned in the proof sketch, to prove Theorem~\ref{Regret: Optimal DP-UCB}, instead of upper bounding the expected number of pulls of a suboptimal arm $j$ directly, we upper bound the expected number of epochs processed by the end of round $T$. 
Recall that $n_j(T)$ denotes the number of pulls of arm $j$ by the end of round $T$ and $r_j(T)$ denotes the index of the most recently processed epoch by the end of round $T$.  Recall $d_j = \left\lceil \log \left(   \frac{24 \ln (T)}{\Delta_j \cdot \min \{\Delta_j, \epsilon\}}  \right) \right\rceil $ and $\omega_j^{(r)} =  \sum_{q = 0}^{r} 2^{q}  =  2^{r+1}-1$.
We have  the number of pulls of a suboptimal arm $j$ by the end of round $T$ is at most
\begin{equation}
\begin{array}{lll}
n_j(T) & \le & \sum\limits_{q = 0}^{r_j(T)+1} 2^q \le \sum\limits_{q = 0}^{d_j+1} 2^q   + \sum\limits_{q = d_j + 1}^{r_j(T)+1} 2^q
 =  \omega_j^{(d_j+1)} + \sum\limits_{q = d_j + 1}^{r_j(T)+1} 2^q \quad.
\end{array}
\label{Ni 2}
\end{equation}
We now prove Theorem~\ref{Regret: Optimal DP-UCB}.
\begin{proof}(of Theorem~\ref{Regret: Optimal DP-UCB})
By taking the expectation of both sides in (\ref{Ni 2}), we have
\begin{equation}
\begin{array}{lll}
 \mathbb{E} \left[ n_j(T)  \right]
 & \le &  \omega_j^{(d_j+1)}  + \sum\limits_{s = d_j + 1}^{\log (T)} \mathbb{P} \left\{ r_j(T) = s  \right\} \cdot \sum\limits_{q = d_j +1}^{s+1} 2^q  \\
 & \le &  \omega_j^{(d_j)}  + 
 \sum\limits_{s = d_j + 1}^{\log (T)} \underbrace{\mathbb{P} \left\{ r_j(T) = s  \right\} \cdot \omega_j^{(s+1)} }_{= O(1), \text{ Lemma~\ref{Tomorrow2}}} \\
  & \le &  O \left(  \frac{\log(T)}{\Delta_j \cdot \min \left\{  \Delta_j, \epsilon\right\}} \right) + O \left( \log(T) \right)\\
  & = & O \left(  \frac{\log(T)}{\Delta_j \cdot \min \left\{  \Delta_j, \epsilon\right\}} \right) \quad,
 \end{array}
 \label{Ni}
\end{equation}
where the second inequality of (\ref{Ni}) uses $\sum_{q = d_j + 1}^{s+1} 2^q \le \omega_j^{(s+1)}$. 

Now, we have the instance-dependent regret of Anytime-Lazy-UCB is
\begin{equation}
\begin{array}{l}
\mathcal{R}(T)  =  \sum\limits_{j: \Delta_j > 0}  \mathbb{E} \left[ n_j(T) \right] \cdot \Delta_j 
\le \sum\limits_{j: \Delta_j > 0} O \left(  \frac{\log(T)}{\Delta_j \cdot \min \left\{  \Delta_j, \epsilon\right\}} \right) \cdot \Delta_j 
=\sum\limits_{j: \Delta_j > 0} O \left(  \frac{\log(T)}{ \min \left\{\Delta_j,  \epsilon\right\}} \right),
\end{array}
\end{equation}
which concludes the proof.
\end{proof}
\begin{lemma}
For any $s \ge d_j+1$, we have $\mathbb{P} \left\{ r_j(T) = s  \right\} \cdot  \omega_j^{(s+1)} \le O(1)$.
\label{Tomorrow2}
\end{lemma}
\begin{proof}(of Lemma~\ref{Tomorrow2})
Recall that $d_j = \left\lceil \log \left(   \frac{24\ln (T)}{\Delta_j \cdot \min (\Delta_j, \epsilon)}  \right) \right\rceil $ and $\omega_j^{(r)} = \sum\limits_{q = 0}^{r} 2^{q} = 2^{r+1} - 1$. Then, we have  $   \frac{24\ln (T)}{\Delta_j \cdot \min (\Delta_j, \epsilon)}    \le 2^{d_j} \le 2 \cdot \frac{24\ln (T)}{\Delta_j \cdot \min (\Delta_j, \epsilon)}   $.
Now, we have
\begin{equation}
\begin{array}{ll}
   & \mathbb{P} \left\{ r_j(T) = s \right\} \cdot \omega_j^{(s+1)} \\ 
   \le & \mathbb{P} \left\{ \exists t \in \{ \omega_j^{(s-1)}+1 , \dotsc, T \}:  J_t = j,  r_j(t-1) = s-1 \right\} \cdot \omega_j^{(s+1)}  \\
   \le &\sum\limits_{t =  \omega_j^{(s-1)}+1}^{T}  \mathbb{P} \left\{  \overline{\mu}_j (t) \ge \overline{\mu}_1(t),  r_j(t-1) = s-1\right\}  \cdot \omega_j^{(s+1)} \\
   \le & \sum\limits_{t =  \omega_j^{(s-1)}+1}^{T}  \mathbb{E} \left[\bm{1} \left\{  \overline{\mu}_j (t) \ge \overline{\mu}_1(t),  r_j(t-1) = s-1 \right\}  \right] \cdot \omega_j^{(s+1)} \\
  \le & \sum\limits_{t =  \omega_j^{(s-1)}+1}^{T} \sum\limits_{\tau = 0}^{ \log(t-1) }  \mathbb{E} \left[  \bm{1} \left\{ \underbrace{ \widetilde{ \mu}_{j, 2^{s-1}} + \sqrt{\frac{\ln (t^{3})}{2^{s-1}}} + \frac{\ln (t^{3})}{\epsilon  2^{s-1}} \ge \widetilde{ \mu}_{1, 2^{\tau}} + \sqrt{\frac{\ln (t^{3})}{2^{\tau}}} + \frac{\ln (t^{3})}{\epsilon  2^{\tau}} }_{\psi} \right\}  \right] \omega_j^{(s+1)}.
 \end{array}
 \label{girl -1}
\end{equation}
If $\psi $ is true, then at least one of the following is true:
\begin{equation}
\begin{array}{lll}
 \widetilde{ \mu}_{j, 2^{s-1}} &\ge & \mu_j + \sqrt{\frac{\ln (t^{3})}{2^{s-1}}} + \frac{\ln (t^{3})}{\epsilon 2^{s-1}} \quad,\\ 
 \widetilde{ \mu}_{1, 2^{\tau}} & \le &\mu_1- \sqrt{\frac{\ln (t^{3})}{2^{\tau}}} - \frac{\ln (t^{3})}{\epsilon 2^{\tau}} \quad,\\  
\frac{\Delta_j}{2} & < & \sqrt{\frac{\ln (t^{3})}{2^{s-1}}} + \frac{\ln (t^{3})}{\epsilon 2^{s-1}} \quad.
\end{array}
\label{Infamous}
\end{equation}
For the first two events in (\ref{Infamous}), we apply concentration inequality of Laplace distribution (Fact~\ref{fact 1}) and  Hoeffding's inequality  (Fact~\ref{Hoeffding}), and have
\begin{equation}
\begin{array}{ll}
& \mathbb{P} \left\{ \widetilde{ \mu}_{j, 2^{s-1}} \ge \mu_j + \sqrt{\frac{\ln (t^{3})}{2^{s-1}}} + \frac{\ln (t^{3})}{\epsilon 2^{s-1}} \right\} \\
\le &\underbrace{ \mathbb{P} \left\{ \widetilde{ \mu}_{j, 2^{s-1}} \ge \widehat{ \mu}_{j, 2^{s-1}} + \frac{\ln (t^3)}{\epsilon 2^{s-1}}  \right\} }_{\text{Fact~\ref{fact 1}}} +  \underbrace{\mathbb{P} \left\{ \widehat{ \mu}_{j, 2^{s-1}} \ge \mu_j + \sqrt{\frac{\ln (t^3)}{2^{s-1}}}  \right\}}_{\text{Fact~\ref{Hoeffding}}} \\
= &O\left(\frac{1}{t^3}\right)\quad.
\end{array}
\label{girl 3}
\end{equation}
Similarly, we have
$\mathbb{P} \left\{ \widetilde{ \mu}_{1, 2^{\tau}} \le \mu_1- \sqrt{\frac{\ln (t^{3})}{2^{\tau}}} - \frac{\ln (t^{3})}{\epsilon 2^{\tau}} \right\} = O\left(\frac{1}{t^3} \right) $.

We now prove that the last event in  (\ref{Infamous}) cannot occur by using contradiction. We have
\begin{equation}
\begin{array}{lll}
 \sqrt{\frac{\ln (t^{3})}{2^{s-1}}} + \frac{\ln (t^{3})}{\epsilon 2^{s-1}}  
& \le & \sqrt{\frac{\ln (T^{3})}{2^{d_j}}} + \frac{\ln (T^{3})}{\epsilon 2^{d_j}}  \quad \text{(Note that $s-1 \ge d_j$)}\\
& \le & \sqrt{\frac{\ln (T^{3})}{ \frac{24 \ln (T)}{\Delta_j \cdot \min \left\{\Delta_j, \epsilon \right\}} }} + \frac{\ln (T^{3})}{\epsilon \cdot \frac{24 \ln (T)}{\Delta_j \cdot \min \left\{\Delta_j, \epsilon \right\}} } \quad \text{(Use the lower bound of $2^{d_j}$)}  \\
& < &  0.5 \Delta_j \quad, 
\end{array}
\end{equation}
which implies  the last event in (\ref{Infamous}) cannot occur.

We now come back to (\ref{girl -1}) and have
\begin{equation}
\begin{array}{ll}
  & \mathbb{P} \left\{ r_j(T) = s \right\} \omega_j^{(s+1)}  \\
  \le & \sum\limits_{t =  \omega_j^{(s-1)}+1}^{T} \sum\limits_{\tau = 0}^{ \log(t-1) }  \mathbb{E} \left[ \underbrace{ \bm{1} \left\{ \widetilde{ \mu}_{j, 2^{s-1}} + \sqrt{\frac{\ln (t^{3})}{2^{s-1}}} + \frac{\ln (t^{3})}{\epsilon 2^{s-1}} \ge \widetilde{ \mu}_{1, 2^{\tau}} + \sqrt{\frac{\ln (t^{3})}{2^{\tau}}} + \frac{\ln (t^{3})}{\epsilon 2^{\tau}} \right\}  }_{\psi}\right] \omega_j^{(s+1)}  \\
   \le & \sum\limits_{t =  \omega_j^{(s-1)}+1}^{T} \sum\limits_{\tau = 0}^{ \log(t-1) }  \left( \underbrace{ \mathbb{P} \left\{ \widetilde{ \mu}_{j, 2^{s-1}} \ge \mu_j + \sqrt{\frac{\ln (t^{3})}{2^{s-1}}} + \frac{\ln (t^{3})}{\epsilon 2^{s-1}} \right\} }_{(=O(1/t^3))}+  \underbrace{\mathbb{P} \left\{ \widetilde{ \mu}_{1, 2^{\tau}} \le \mu_1- \sqrt{\frac{\ln (t^{3})}{2^{\tau}}} - \frac{\ln (t^{3})}{\epsilon 2^{\tau}} \right\} }_{(=O(1/t^3))}\right)\omega_j^{(s+1)} \\
  
 \le & \sum\limits_{t =  \omega_j^{(s-1)}+1}^{T} \sum\limits_{\tau = 0}^{\log(t-1) } O \left( \frac{1}{t^3} \right) \cdot \omega_j^{(s+1)}  \\
  \le & \sum\limits_{t =  \omega_j^{(s-1)}+1}^{T} O \left( \frac{1}{t^2} \right) \cdot \omega_j^{(s+1)}  \\
  \le & \int_{\omega_j^{(s-1)}}^{T} O \left( \frac{1}{t^2} \right) dt \cdot \omega_j^{(s+1)} \\
  < & O \left(  \frac{\omega_j^{(s+1)} }{\omega_j^{(s-1)} } \right) \\
   \le & O(1) \quad,
 \end{array}
\end{equation}
which concludes the proof.\end{proof}
             
\subsubsection{Proofs for the instance-independent regret bound} \label{app: isolated 2}

Let $\Delta^* := \sqrt{K\ln(T)/T}$ be the critical gap. We have the regret 
\begin{equation}
    \begin{array}{lll}
      \mathcal{R}(T) & =  &\sum\limits_{t=1}^{T} \mathbb{E}\left[\Delta_{J_t} \right]  \\
 &   \le     & T\Delta^* + \sum_{j: \Delta_j > \Delta^*} \mathbb{E}\left[n_j(T) \right] \cdot \Delta_{j}\\
 & \le & \sqrt{KT\ln(T)} + O \left( \frac{K\ln(T)}{\sqrt{K\ln(T)/T}} \right)+  O \left(\frac{K\ln(T)}{\epsilon} \right)\\
 & \le & O \left(\sqrt{KT\ln(T)} \right) + O \left(\frac{K\ln(T)}{\epsilon} \right)\quad,
    \end{array}
\end{equation}
which concludes the proof.

 \section{Proofs for Lazy-DP-TS} \label{app: Lazy-DP-TS}

 The full proof for Theorem~\ref{Regret: Lazy-DP-TS} is shown in Section~\ref{app: ts}.
 
 \subsection{Proofs for Theorem~\ref{Regret: Lazy-DP-TS}} \label{app: ts}

Here, we only show the full proof for the instance-dependent regret bound. For the instance-independent regret bound, we can reuse the same proof as already shown in Section~\ref{app: isolated 2}.

We first recall some definitions and notation presented in Section~\ref{sec: lazy-ts}. 
 Recall that $O_j(t-1)$ is the number of observations used to compute the differentilly private empirical mean of arm $j$ by the end of round $t-1$. From Algorithm~\ref{Optimal DP-TS}, we know that it only takes values from $2^r, r\ge 0$.  Recall that $y_j = \mu_1 - \frac{\Delta_j}{6}$, and  $E_j^{\theta}(t)$ is the event  that $\left\{\theta_j(t) \le y_j \right\}$ and $\overline{E_j^{\theta}(t)}$ is the complementary event of $E_j^{\theta}(t)$. Recall that $C_j(t-1)$ is the good event that the confidence interval holds, that is,  $\left\{\left| \mu_j -\widehat{\mu}_{j, O_j(t-1)} \right| \le \sqrt{\frac{3\ln(t)}{O_j(t-1)}}\right\}$, and  $\overline{C_j(t-1)}$ is the complementary event of $C_j(t-1)$.  Recall that $G_j(t-1)$ is the good event that the true empirical mean is close to the differentially private empirical mean, that is, $ \left\{ \left| \widehat{\mu}_{j, O_j(t-1)} - \widetilde{\mu}_{j, O_j(t-1)} \right|  \le \frac{3\ln(t)}{\epsilon \cdot O_j(t-1)} \right\}$, and  $\overline{G_j(t-1)}$ is the complementary event of $G_j(t-1)$. 
Also, recall the regret decomposition shown in (\ref{cloud 22}) is 
\begin{equation}
    \begin{array}{lll}
 \mathbb{E} \left[n_j(T) \right] & = & \sum\limits_{t = K+1}^{T}\mathbb{E} \left[ \bm{1} \left\{ J_t = j \right\}  \right] +1 \\
 
 &   \le & \underbrace{\sum\limits_{t = K+1}^{T}\mathbb{P}  \left\{J_t = j, C_j(t-1), G_j(t-1), \overline{E_j^{\theta}(t)} \right\}}_{=:\omega_1} 
   +  \underbrace{\sum\limits_{t = K+1}^{T} \mathbb{P}  \left\{J_t = j,   E_j^{\theta}(t), G_1(t-1) \right\}}_{=:\omega_2} \\
   & + & \underbrace{\sum\limits_{t = K+1}^{T} \mathbb{P}   \left\{ \overline{C_j(t-1)} \right\} +  \mathbb{P} \left\{ \overline{G_j(t-1)} \right\} +  \mathbb{P} \left\{ \overline{G_1(t-1)} \right\}}_{=: \omega_3} +1\quad.
    \end{array}
\end{equation}
\paragraph{Upper bounding $\omega_1$.} We prepare Lemma~\ref{lazy dp-ts high term2} below to bound this term, which states that if both of the good events occur, it is highly unlikely that the posterior sample of a suboptimal arm $j$ will be drawn greater than $y_j$, that is, $\overline{E_j^{\theta}(t)}$ is a low probability event. For the proof, only concentration bounds are needed.

\begin{lemma}
For any suboptimal arm $j$, we have
\begin{equation}
\begin{array}{lll}
  \underbrace{  \sum\limits_{t =K+ 1}^{T} \mathbb{E} \left[ \bm{1} \left\{J_t = j, C_j(t-1), G_j(t-1) , \overline{E_j^{\theta}(t)} \right\}  \right]}_{\omega_1} &\le &O \left( \frac{\ln(T)}{\Delta_j \cdot \min \left\{\epsilon, \Delta_j  \right\}}\right) \quad.
   \end{array}
   \label{sunshine 3}
\end{equation}
\label{lazy dp-ts high term2}
\end{lemma}
\begin{proof}(of Lemma~\ref{lazy dp-ts high term2})
  Let $\lambda^{(j)}_r$ be the round such that, by the end of this round, we use $2^r$ fresh observations to update arm $j$'s differentially private empirical mean.
  Note that in all rounds $t \in \left\{ \lambda^{(j)}_r+1, \lambda^{(j)}_r +2, \dotsc, \lambda^{(j)}_{r+1} \right\}$, the number of observations for arm $j$ is exactly $2^r$ and the differentially private empirical mean for arm $j$ stays the same. 
Recall $d_j =\left\lceil \log \left( \frac{72  \ln(T)}{\Delta_j \cdot \min \left\{ \epsilon, \Delta_j \right\}} \right) \right\rceil$. We have
 
 \begin{equation}
    \begin{array}{lll}
         \omega_1 & = & \sum\limits_{t=K+1}^{T} \mathbb{E} \left[  \bm{1} \left\{J_t = j ,   C_j(t-1),G_j(t-1), \overline{E_j^{\theta}(t)}   \right\}\right] \\

        & \le  & \sum\limits_{r=0}^{\log(T)}  \mathbb{E} \left[ \sum\limits_{t=\lambda^{(j)}_r +1}^{\lambda^{(j)}_{r+1}}   \bm{1} \left\{J_t = j ,   C_j(t-1),G_j(t-1) ,  \overline{E_j^{\theta}(t)}   \right\}\right] \\
     &    \le & \sum\limits_{r=0}^{ d_j }  \mathbb{E} \left[ \sum\limits_{t=\lambda^{(j)}_r +1}^{\lambda^{(j)}_{r+1}}   \bm{1} \left\{ J_t = j   \right\}\right] 
         + \sum\limits_{r=d_j +1}^{\log(T)}  \mathbb{E} \left[ \sum\limits_{t=\lambda^{(j)}_r +1}^{\lambda^{(j)}_{r+1}}   \bm{1} \left\{ J_t = j ,   C_j(t-1),G_j(t-1) , \overline{E_j^{\theta}(t)}   \right\}\right] \\
       &   \le & \sum\limits_{r=0}^{ d_j } 2^{r+1} + \sum\limits_{r= d_j +1}^{\log(T)}  \mathbb{E} \left[ \sum\limits_{t=\lambda^{(j)}_r +1}^{\lambda^{(j)}_{r+1}}   \bm{1} \left\{ J_t = j ,   C_j(t-1),G_j(t-1) , \overline{E_j^{\theta}(t)}   \right\}\right] \\
       &   \le & O \left( \frac{\log(T)}{\Delta_j \cdot \min \left\{\epsilon, \Delta_j  \right\}} \right) + \sum\limits_{r=d_j+1 }^{\log(T)}\underbrace{  \mathbb{E} \left[ \sum\limits_{t=\lambda^{(j)}_r +1}^{\lambda^{(j)}_{r+1}}   \bm{1} \left\{ J_t = j ,   C_j(t-1),G_j(t-1) , \overline{E_j^{\theta}(t)}   \right\}\right]}_{=:\eta} \quad.
         \end{array}
         \label{temp 3}
         \end{equation}
 We now show $\eta \le O \left( 1/T\right)$. We  have     \begin{equation}
         \begin{array}{lll}
             \eta & = &   \mathbb{E} \left[ \sum\limits_{t=\lambda^{(j)}_r +1}^{\lambda^{(j)}_{r+1}}   \bm{1} \left\{ J_t = j ,   C_j(t-1),G_j(t-1) , \overline{E_j^{\theta}(t)}   \right\}\right]\\
              & \le &   \mathbb{E} \left[\sum\limits_{t=\lambda^{(j)}_r +1}^{\lambda^{(j)}_{r+1}} \bm{1} \left\{ \overline{\mu}_{j, O_j(t-1)}(t) \le \mu_j + \frac{\Delta_j}{3}, O_j(t-1) = 2^r \right\} \cdot  \bm{1} \left\{J_t = j ,  \theta_j(t) > y_j   \right\}\right]\\
                  & \le & \sum\limits_{t=2^r +1}^{T}  \mathbb{E} \left[ \underbrace{\underbrace{\bm{1} \left\{ \overline{\mu}_{j, O_j(t-1)}(t) \le \mu_j + \frac{\Delta_j}{3}, O_j(t-1) = 2^r \right\}}_{\chi} \cdot \mathbb{P} \left\{\theta_j(t) > y_j \mid \mathcal{F}_{t-1}  \right\}}_{\gamma}\right]\quad.
             \end{array}
             \label{temp 4}
             \end{equation}
             The first inequality in (\ref{temp 4} ) uses the fact that  if   events $C_j(t-1)$ and $G_j(t-1)$
are true simultaneously, for any $t \in \left\{\lambda^{(j)}_r +1, \dotsc, \lambda^{(j)}_{r +1} \right\}$, we have
 $
\overline{\mu}_{j, O_j(t-1)}(t) 
   =  \text{Clip}_{[0,1]} \left( \widetilde{\mu}_{j, O_j(t-1)}  +  \frac{3\ln(t)}{\epsilon \cdot O_j(t-1) } \right) 
   \le  \text{Clip}_{[0,1]} \left(\mu_j + \frac{6\ln(t)}{\epsilon \cdot \frac{72 \ln(t)}{\epsilon \cdot \Delta_j}} + \sqrt{\frac{3\ln(t)}{\frac{72 \ln(t)}{\Delta_j^2}}} \right) \le  \mu_j + \frac{\Delta_j}{3} $ due to $O_j(t-1) = 2^r \ge  \frac{72  \ln(T)}{\Delta_j \cdot \min \left\{ \epsilon, \Delta_j \right\}}$.
   
To continue upper bounding $\eta$, we  categorize all instantiations $F_{t-1}$ of $\mathcal{F}_{t-1}$ into two types based on whether the indicator function $\chi$ returns 1.   Let $F^{\text{Beta}}_{\alpha, \beta}(\cdot)$ be the CDF of a Beta distribution with parameters $\alpha$ and $\beta$ and let $F^{B}_{n,p}(\cdot)$ be the CDF of a Binomial distribution with parameters $n$ and $p$.
  For the $F_{t-1}$ such that the indicator function $\chi$ returns 0, we have $\gamma = 0$.
 For the    $F_{t-1}$ such that  the indicator function $\chi$ returns 1, we have
\begin{equation}
    \begin{array}{lll}
\gamma   & = & \mathbb{P} \left\{ \theta_j(t) > y_j \mid \mathcal{F}_{t-1} = F_{t-1} \right\} \\
   & = & 1- F^{\text{Beta}}_{\overline{\mu}_{j, O_j(t-1)}(t) \cdot O_j(t-1) +1, \    \left(1-\overline{\mu}_{j, O_j(t-1)}(t)\right) \cdot O_j(t-1)  + 1}(y_j) \\
    & = & 1- F^{\text{Beta}}_{\overline{\mu}_{j, O_j(t-1)}(t) \cdot 2^r +1, \    \left(1-\overline{\mu}_{j, O_j(t-1)}(t)\right) \cdot 2^r  + 1}(y_j) \\
    & {\le}^{(a)} & 1- F^{\text{Beta}}_{\left(\mu_j + \frac{\Delta_j}{3}   \right)\cdot 2^r + 1, \ \left(1-\left(\mu_j + \frac{\Delta_j}{3}  \right)\right) \cdot 2^r + 1 }(y_j) \\
     &  \le^{(b)} & 1- F^{\text{Beta}}_{ \left\lceil\left(\mu_j + \frac{\Delta_j}{3}   \right)\cdot 2^r\right\rceil + 1, \ 2^r- \left\lceil\left(\mu_j + \frac{\Delta_j}{3}  \right) \cdot 2^r \right\rceil + 1 }(y_j) \\
   &  {=}^{(c)} & F^{B}_{2^r+1, y_j}\left(\left\lceil\left(\mu_j + \frac{\Delta_j}{3} \right) \cdot 2^r \right\rceil \right) \\
    & \le^{(d)} & F^{B}_{2^r+1, y_j}\left(\left(\mu_j + \frac{\Delta_j}{3} \right) \cdot 2^r + 1 \right) \\
    &    = & F^{B}_{2^r+1, y_j}\left(\left(\mu_j + \frac{\Delta_j}{3} + \frac{1}{2^r}\right) \cdot 2^r  \right) \\
   &             {\le}^{(e)} & F^{B}_{2^r+1, y_j}\left(\left(\mu_j + \frac{\Delta_j}{3} + \frac{\Delta_j}{12}\right) \cdot \left(2^r+1\right)  \right) \\
   &  {\le}^{(f)} & e^{- \left(2^r + 1  \right) \cdot d_{\text{KL}} \left( \mu_j + \frac{5\Delta_j}{12}, y_j \right)} \\
    &   {\le}^{(g)} & e^{- 2^{d_j} \cdot 2 \cdot \frac{25\Delta_j^2}{144}} \\
   &  \le & \frac{1}{T^2} \quad.
    \end{array}
    \label{temp 6}
\end{equation}
Here, we  provide explanations about each key step in (\ref{temp 6}). 
Inequalities (a) and (b) use the facts that for any fixed $n \ge 1$,  $\text{Beta}\left(\alpha + 1, n-\alpha+1 \right)$ first-order stochastic dominates $\text{Beta}\left(\alpha' + 1, n-\alpha'+1 \right)$ when $n \ge \alpha \ge \alpha' \ge 0$. 
Equality (c) uses the relationship between the CDFs of Beta distributions and Binomial distributions (Fact~\ref{cdfs}), and  inequality (d) uses the non-decreasing property of a CDF. Inequality (e) uses $\frac{1}{2^r} \le \frac{1}{2^{d_j} } < \frac{\Delta_j}{12}$.
Inequality (f) uses the Chernoff-Hoeffding bound (Fact~\ref{chernoff-hoeffding}), and inequality (g) uses Pinsker's inequality, i.e., $d_{\text{KL}}(x,y) \ge 2(x-y)^2$.

By plugging in the upper bound of $\gamma$ into (\ref{temp 4}), we have
$\eta \le 1/T$, which implies $\omega_1 \le O \left( \frac{\ln(T)}{\Delta_j \cdot \min \left\{\epsilon, \Delta_j  \right\}} \right) + 1$.
\end{proof}

\paragraph{Upper bounding $\omega_2$.} We prepare Lemma~\ref{lazy dp-ts high term} below to bound this term, which states that if the optimal arm's differentially private empirical mean is very close to the true empirical mean and suboptimal arm $j$'s posterior sample is distributed close to the true mean, 
then the learning agent only pulls this suboptimal arm $j$ a small number of times. For the proof, both concentration and anti-concentration bounds are needed.

 \begin{lemma}
For any suboptimal arm $j$, we have
\begin{equation}
\begin{array}{lll}
\underbrace{   \sum\limits_{t = K+1}^{T} \mathbb{E} \left[\bm{1}  \left\{J_t = j,   E_j^{\theta}(t), G_1(t-1) \right\}\right]}_{\omega_2} &\le & O \left( \frac{\log(T)}{\Delta_j^2}\right)  \quad.
    \end{array}   
    \label{sunshine 2}
\end{equation}
\label{lazy dp-ts high term}
\end{lemma}

\begin{proof}(of Lemma~\ref{lazy dp-ts high term})     Recall that $\mathcal{F}_{t-1}$ collects all the history information by the end of round $t-1$ consisting of  the pulled arms, the rewards associated with the pulled arms, and the injected noise. Let $\lambda^{(1)}_r$ be the round such that by this round we use $2^r$ fresh observations to update the optimal arm $1$'s differentially private empirical mean. Based on the definition of $\lambda^{(1)}_r$, we know that in all rounds $t \in \left\{ \lambda^{(1)}_r+1, \lambda^{(1)}_r +2, \dotsc, \lambda^{(1)}_{r+1} \right\}$, the number of observations for arm $1$ is exactly $2^r$ and the differentially private empirical mean for arm $1$ stays the same in these rounds.  According to Algorithm~\ref{Optimal DP-TS}, we know $\lambda^{(1)}_0  \le K$. We have
    \begin{equation}
 \begin{array}{lll}
\omega_2 & = & \mathbb{E} \left[ \sum\limits_{t=K+1}^{T} \bm{1} \left\{J_t = j  , E^{\theta}_j(t) , G_1(t-1)  \right\}\right] \\

&= & \mathbb{E}\left[\sum\limits_{t=K+1}^{T}  \underbrace{\mathbb{P}   \left\{J_t = j , E^{\theta}_j(t) ,G_1(t-1) \mid \mathcal{F}_{t-1}\right\}}_{\text{ Lemma~\ref{Martingale event}}} \right] \\
           &   \le & \mathbb{E}\left[\sum\limits_{t=K+1}^{T} \underbrace{ \frac{\mathbb{P}   \left\{\theta_1(t) \le y_j  \mid \mathcal{F}_{t-1}\right\}  }{\mathbb{P}   \left\{\theta_1(t) > y_j \mid \mathcal{F}_{t-1} \right\} } \cdot \mathbb{P}   \left\{J_t = 1 ,  E^{\theta}_j(t) ,G_1(t-1) \mid \mathcal{F}_{t-1} \right\}}_{\text{Lemma~\ref{Martingale event}}} \right]  \\
                &   \le & \mathbb{E}\left[\sum\limits_{t=K+1}^{T} \mathbb{E} \left[ \frac{\mathbb{P}   \left\{\theta_1(t) \le y_j  \mid \mathcal{F}_{t-1}\right\}  }{\mathbb{P}    \left\{\theta_1(t) > y_j \mid \mathcal{F}_{t-1} \right\} } \cdot \bm{1}   \left\{J_t = 1   , G_1(t-1)  \right\} \mid \mathcal{F}_{t-1} \right] \right]  \\
                    &\le & \mathbb{E}\left[\sum\limits_{t=K+1}^{T}  \frac{\mathbb{P}   \left\{\theta_1(t) \le y_j  \mid \mathcal{F}_{t-1}\right\}  }{\mathbb{P}   \left\{\theta_1(t) > y_j \mid \mathcal{F}_{t-1} \right\} } \cdot \bm{1}   \left\{J_t = 1 ,  G_1(t-1)  \right\} \right]  \\
          &    \le &\sum\limits_{r=0}^{\log(T)} \mathbb{E}\left[ \sum\limits_{t = \lambda^{(1)}_r + 1}^{\lambda^{(1)}_{r+1}}  \underbrace{\frac{\mathbb{P}   \left\{\theta_1(t) \le y_j   \mid \mathcal{F}_{t-1}\right\}  }{\mathbb{P}   \left\{\theta_1(t) > y_j \mid \mathcal{F}_{t-1} \right\} }  \cdot \bm{1} \left\{G_1(t-1)  \right\}}_{=:\rho} \cdot \bm{1}  \left\{J_t = 1 \right\} \right] \quad.
          \end{array}
          \label{snow 1}
          \end{equation}
 Let $\theta'_{1, O_1(t-1)} \sim \text{Beta} \left(\widehat{\mu}_{1,O_1(t-1)} \cdot O_1(t-1) +1, (1-\widehat{\mu}_{1,O_1(t-1)}) \cdot O_1(t-1) +1 \right)$. We now start to link the coefficient $\frac{\mathbb{P}   \left\{\theta_1(t) \le y_j   \mid \mathcal{F}_{t-1} = F_{t-1}\right\}  }{\mathbb{P}   \left\{\theta_1(t) > y_j \mid \mathcal{F}_{t-1} = F_{t-1} \right\} }$   to  
 $\frac{\mathbb{P}   \left\{\theta'_{1, O_1(t-1)} \le y_j   \mid \mathcal{F}_{t-1} = F_{t-1}\right\}  }{\mathbb{P}   \left\{\theta'_{1, O_1(t-1)} > y_j \mid \mathcal{F}_{t-1} = F_{t-1} \right\} }$ for any possible instantiation $F_{t-1}$ of $\mathcal{F}_{t-1}$.
  We can categorize all  the instantiations $F_{t-1}$ of $\mathcal{F}_{t-1}$ into two types based on whether $\bm{1} \left\{G_1(t-1)\right\}$ returns 1 or not.
              For the instantiation $F_{t-1}$ such that $\bm{1} \left\{G_1(t-1)\right\} = 0$, we have $\rho = 0$. Note that $\frac{\mathbb{P}   \left\{\theta_1(t) \le y_j   \mid \mathcal{F}_{t-1} = F_{t-1}\right\}  }{\mathbb{P}   \left\{\theta_1(t) > y_j \mid \mathcal{F}_{t-1}=F_{t-1} \right\} } < \infty$. 
              For the instantiation $F_{t-1}$ such that $\bm{1} \left\{G_1(t-1)\right\} = 1$, we have $   \frac{\mathbb{P}   \left\{\theta_1(t) \le y_j   \mid \mathcal{F}_{t-1} = F_{t-1}\right\}  }{\mathbb{P}   \left\{\theta_1(t) > y_j \mid \mathcal{F}_{t-1}=F_{t-1} \right\} }  \le \frac{\mathbb{P}   \left\{\theta'_{1, O_1(t-1)} \le y_j   \mid \mathcal{F}_{t-1} = F_{t-1}\right\}  }{\mathbb{P}   \left\{\theta'_{1, O_1(t-1)} > y_j \mid \mathcal{F}_{t-1}=F_{t-1} \right\} }$.  This is because, if event $G_1(t-1)$ is true, we have $\widetilde{\mu}_{1,O_1(t-1)} + \frac{3\ln(t)}{\epsilon \cdot O_1(t-1)} \ge \widehat{\mu}_{1,O_1(t-1)} \ge 0 $, which implies $\overline{\mu}_{1, O_1(t-1)} (t) = \text{Clip}_{[0,1]}\left(\widetilde{\mu}_{1,O_1(t-1)} + \frac{3\ln(t)}{\epsilon \cdot O_1(t-1)} \right) \ge \widehat{\mu}_{1,O_1(t-1)} $. Since $\overline{\mu}_{1, O_1(t-1)} (t) \in [0,1]$ and $\overline{\mu}_{1, O_1(t-1)} (t) \ge \widehat{\mu}_{1,O_1(t-1)}$, we have the fact that
               $\text{Beta} \left(\overline{\mu}_{1,O_1(t-1)}(t) \cdot O_1(t-1) +1, \left(1-\overline{\mu}_{1,O_1(t-1)}(t)\right) \cdot O_1(t-1) +1 \right)$ first-order stochastically dominates  
          $\text{Beta} \left(\widehat{\mu}_{1,O_1(t-1)} \cdot O_1(t-1) +1, (1-\widehat{\mu}_{1,O_1(t-1)}) \cdot O_1(t-1) +1 \right)$.  
          
          Now, we have
        \begin{equation}
        \begin{array}{lll}
   \omega_2 &\le &\sum\limits_{r=0}^{\log(T)}\mathbb{E}\left[ \sum\limits_{t = \lambda^{(1)}_r + 1}^{\lambda^{(1)}_{r+1}}  \frac{\mathbb{P}   \left\{\theta_1(t) \le y_j   \mid \mathcal{F}_{t-1}\right\}  }{\mathbb{P}   \left\{\theta_1(t) > y_j \mid \mathcal{F}_{t-1} \right\} }  \cdot \bm{1} \left\{G_1(t-1)  \right\} \cdot \bm{1}  \left\{J_t = 1 \right\} \right] \\
     &&\\
    & \le &\sum\limits_{r=0}^{\log(T)}\mathbb{E}\left[ \sum\limits_{t = \lambda^{(1)}_r + 1}^{\lambda^{(1)}_{r+1}}  \frac{\mathbb{P}   \left\{\theta'_{1, O_1(t-1)}  \le y_j   \mid \mathcal{F}_{t-1}\right\}  }{\mathbb{P}   \left\{\theta'_{1, O_1(t-1)} > y_j \mid \mathcal{F}_{t-1} \right\} } \cdot \bm{1}  \left\{J_t = 1 \right\} \right] \\
    &&\\
    & = &\sum\limits_{r=0}^{\log(T)}\mathbb{E}\left[ \sum\limits_{t = \lambda^{(1)}_r + 1}^{\lambda^{(1)}_{r+1}}  \frac{\mathbb{P}   \left\{\theta'_{1, 2^r}  \le y_j   \mid \mathcal{F}_{\lambda_r}^{(1)}\right\}  }{\mathbb{P}   \left\{\theta'_{1, 2^r}  > y_j \mid \mathcal{F}_{\lambda^{(1)}_r} \right\} } \cdot \bm{1}  \left\{J_t = 1 \right\} \right] \\
      &&\\
     & = &\sum\limits_{r=0}^{\log(T)} \mathbb{E}\left[ \frac{\mathbb{P}   \left\{\theta'_{1, 2^r} \le y_j   \mid \mathcal{F}_{\lambda^{(1)}_r}\right\}  }{\mathbb{P}   \left\{\theta'_{1, 2^r} > y_j \mid \mathcal{F}_{\lambda^{(1)}_r} \right\} } \cdot \underbrace{\sum\limits_{t = \lambda^{(1)}_r + 1}^{\lambda^{(1)}_{r+1}}  \bm{1}  \left\{J_t = 1 \right\}}_{\le 2^{r+1}} \right] \\
       &&\\
           &     \le &  \sum\limits_{r=0}^{\log(T)} 2^{r+1} \cdot  \mathbb{E}\left[\frac{\mathbb{P}   \left\{\theta'_{1, 2^r} \le y_j   \mid \mathcal{F}_{\lambda^{(1)}_r}\right\}  }{\mathbb{P}   \left\{\theta'_{1, 2^r} > y_j \mid \mathcal{F}_{\lambda^{(1)}_r} \right\} }  \right] \quad,
             \end{array}
             \label{do not disturb}
        \end{equation}
where the last inequality  uses the fact that the number of pulls of arm $1$ is at most $2^{r+1}$ from round $\lambda^{(1)}_r +1$ up to (and including) round $\lambda^{(1)}_{r+1}$  based on the definition of $\lambda^{(1)}_{r+1}$.
        
We continue upper bounding (\ref{do not disturb}) by using Lemma~2.9 in \cite{agrawal2017near}. Recall $d_1 =  \log \left( \frac{8}{\mu_1 - y_j} \right) $. When $r \le \lfloor d_1 \rfloor$, we have $2^r \le 2^{d_1} \le  \frac{8}{\mu_1 - y_j}$,  which implies the number of observations for arm $1$ is not enough to have a  concentrated posterior distribution.  When $ r \ge \lceil d_1 \rceil$, we have $2^r \ge 2^{d_1} \ge  \frac{8}{\mu_1 - y_j}$, which implies the number of observations  is enough to have a  concentrated posterior distribution.
From Lemma~2.9 in \cite{agrawal2017near}, we have
 \begin{equation}
  \begin{array}{l}
\mathbb{E}\left[\frac{\mathbb{P}   \left\{\theta'_{1,2^r} \le y_j   \mid \mathcal{F}_{\lambda^{(1)}_r}\right\}  }{\mathbb{P}   \left\{\theta'_{1,2^r} > y_j \mid \mathcal{F}_{\lambda^{(1)}_r} \right\} }  \right]  
   \le  \begin{cases}
    \frac{3}{\mu_1 - y_j},&  r \le \left\lfloor d_1 \right\rfloor\\
    \Theta\left( e^{-0.5(\mu_1 - y_j)^2 \cdot  2^{r}} + \frac{e^{-2(\mu_1 - y_j)^2\cdot 2^r} }{(2^r+1)  (\mu_1-y_j)^2} + \frac{1}{e^{0.25(\mu_1-y_j)^2 \cdot 2^r}-1} \right),              & r \ge \left\lceil d_1 \right\rceil.
\end{cases}
     \end{array}
   \label{temp 11}
\end{equation}
By applying the results shown in (\ref{temp 11}) into (\ref{do not disturb}), we have
 \begin{equation}
     \begin{array}{lll}
     
   \omega_2 & \le & \sum\limits_{r=0}^{\log(T)} 2^{r+1} \cdot  \mathbb{E}\left[\frac{\mathbb{P}   \left\{\theta'_{1,2^r}\le y_j   \mid \mathcal{F}_{\lambda^{(1)}_r}\right\}  }{\mathbb{P}   \left\{\theta'_{1,2^r} > y_j \mid \mathcal{F}_{\lambda^{(1)}_r} \right\} }  \right] \\
     & \le & \sum\limits_{r=0}^{\lfloor d_1 \rfloor} 2^{r+1} \cdot \frac{3}{\mu_1 - y_j} +  \sum\limits_{r=\lceil d_1 \rceil}^{\log(T)} 2^{r+1} \cdot \Theta\left( e^{-0.5(\mu_1 - y_j)^2 \cdot  2^{r} } + \frac{e^{-2(\mu_1 - y_j)^2\cdot 2^r} }{(2^r+1) \cdot (\mu_1-y_j)^2} + \frac{1}{e^{0.25(\mu_1-y_j)^2 \cdot 2^r}-1} \right)  \\
     & \le & O \left( \frac{1}{(\mu_1 - y_j)^2} \right) + \sum\limits_{r=\lceil d_1 \rceil}^{\log(T)}\Theta\left(2^r \cdot e^{-(\mu_1 - y_j)^2 \cdot  2^{r} \cdot  \frac{1}{2}} \right)  + O \left(  \frac{\log(T)}{(\mu_1 - y_j)^2}\right) \\
      &\le & O \left(  \frac{\log(T)}{(\mu_1 - y_j)^2}\right)  \\
     & \le & O \left(  \frac{\log(T)}{\Delta_j^2}\right) \quad,
     \end{array}
 \end{equation}
 which concludes the proof.
\end{proof}

Prior to proving Lemma~\ref{Martingale event}, we restate the lemma for convenience.

\paragraph{Restatement of Lemma~\ref{Martingale event}.}
For any suboptimal $j$ and any instantiation $F_{t-1}$ of $\mathcal{F}_{t-1}$, we have \begin{equation}
    \begin{array}{lll}
          \mathbb{P}   \left\{J_t = j ,  E^{\theta}_j(t) , G_1(t-1) \mid \mathcal{F}_{t-1} = F_{t-1} \right\}
    \le  \frac{\mathbb{P}   \left\{\theta_1(t) \le y_j  \mid \mathcal{F}_{t-1} = F_{t-1} \right\}  }{\mathbb{P}   \left\{\theta_1(t) > y_j  \mid \mathcal{F}_{t-1} = F_{t-1}  \right\} }  \mathbb{P}   \left\{J_t = 1 ,   E^{\theta}_j(t), G_1(t-1) \mid \mathcal{F}_{t-1} = F_{t-1}  \right\}. 
    \end{array}
    \label{martingale eq}
\end{equation}
  
 \begin{proof}(of Lemma~\ref{Martingale event})
 The proof is very similar to the proof for Lemma~2.8 \citep{agrawal2017near}. 
The LHS in (\ref{martingale eq}) is
\begin{equation}
    \begin{array}{ll}
    & \mathbb{P}   \left\{J_t = j , E^{\theta}_j(t) , G_1(t-1) \mid \mathcal{F}_{t-1} = F_{t-1}  \right\} \\
    \le & \mathbb{P}   \left\{ G_1(t-1), \theta_i(t) \le y_j, \forall i \in \mathcal{A}  \mid \mathcal{F}_{t-1} = F_{t-1}  \right\} \\
    = & \mathbb{P}   \left\{\theta_1(t) \le y_j, G_1(t-1) \mid \mathcal{F}_{t-1} = F_{t-1} \right\} \cdot \mathbb{P}   \left\{\theta_i(t) \le y_j, \forall i \in \mathcal{A} \backslash \{1\}  \mid \mathcal{F}_{t-1} = F_{t-1}  \right\} \\
    = & \mathbb{E} \left[ \bm{1}  \left\{\theta_1(t) \le y_j, G_1(t-1) \right\} \mid \mathcal{F}_{t-1} = F_{t-1} \right] \cdot \mathbb{P}   \left\{\theta_i(t) \le y_j, \forall i \in \mathcal{A} \backslash \{1\}  \mid \mathcal{F}_{t-1} = F_{t-1}  \right\} \\
    = &\bm{1} \left\{G_1(t-1)\right\} \mathbb{E} \left[ \bm{1}  \left\{\theta_1(t) \le y_j  \right\} \mid \mathcal{F}_{t-1} = F_{t-1} \right] \cdot \mathbb{P}   \left\{\theta_i(t) \le y_j, \forall i \in \mathcal{A} \backslash \{1\}  \mid \mathcal{F}_{t-1} = F_{t-1}  \right\} \\
      = &\bm{1} \left\{G_1(t-1)\right\} \mathbb{P}   \left\{\theta_1(t) \le y_j \mid \mathcal{F}_{t-1} = F_{t-1}  \right\}  \cdot \mathbb{P}   \left\{\theta_i(t) \le y_j, \forall i \in \mathcal{A} \backslash \{1\}  \mid \mathcal{F}_{t-1} = F_{t-1}  \right\} \quad.
    \end{array}
    \label{temp 7}
\end{equation}
We also have
\begin{equation}
    \begin{array}{ll}
    & \mathbb{P}   \left\{J_t = 1 , E^{\theta}_j(t), G_1(t-1)  \mid \mathcal{F}_{t-1} = F_{t-1}  \right\} \\
    \ge & \mathbb{P}   \left\{ G_1(t-1), \theta_1(t) > y_j \ge \theta_i(t), \forall i \in \mathcal{A} \backslash \{1\}  \mid \mathcal{F}_{t-1} = F_{t-1}  \right\} \\
    = & \mathbb{P}   \left\{\theta_1(t) > y_j , G_1(t-1) \mid \mathcal{F}_{t-1} = F_{t-1}  \right\}  \cdot \mathbb{P}   \left\{y_j \ge \theta_i(t), \forall i \in \mathcal{A} \backslash \{1\}  \mid \mathcal{F}_{t-1} = F_{t-1}  \right\}\\
    =& \bm{1} \left\{G_1(t-1)\right\} \mathbb{P}   \left\{\theta_1(t) > y_j \mid \mathcal{F}_{t-1} = F_{t-1}  \right\}  \cdot \mathbb{P}   \left\{\theta_i(t) \le y_j, \forall j \in \mathcal{A} \backslash \{1\}  \mid \mathcal{F}_{t-1} = F_{t-1}  \right\}\quad.
    \end{array}
    \label{temp 8}
\end{equation}
For the instantiation $F_{t-1}$ such that $\bm{1} \left\{G_1(t-1) \right\} = 0$, it is trivial to prove since both sides in (\ref{martingale eq}) are $0$. For the instantiation $F_{t-1}$  such that $\bm{1} \left\{G_1(t-1) \right\} = 1$, by combining (\ref{temp 7}) and (\ref{temp 8}), we  conclude the proof.  Note that for  any $y_j \in (0, 1)$, we have $\mathbb{P} \left\{\theta_1(t) >y_j  \mid \mathcal{F}_{t-1} = F_{t-1}\right\} \in (0,1)$.\end{proof}

\begin{lemma}
For any arm $j \in [K]$, we have
$ \sum\limits_{t =K+ 1}^{T} \mathbb{E} \left[ \bm{1} \left\{ \overline{C_j(t-1)} \right\}  \right] \le O(1) \quad.$
    \label{low prob event 1}
\end{lemma}
\begin{proof}(of Lemma~\ref{low prob event 1})
Using Hoeffding's inequality (Fact~\ref{Hoeffding}), we have
\begin{equation}
    \begin{array}{lll}
     \sum\limits_{t = K+1}^{T} \mathbb{E} \left[ \bm{1} \left\{ \overline{C_j(t-1)} \right\}  \right] 
    &=& \sum\limits_{t = K+1}^{T} \mathbb{P} \left\{ \left| \widehat{\mu}_{j, O_j(t-1)} - \mu_j \right| > \sqrt{\frac{3\ln(t)}{O_j(t-1)}}\right\}  \\
    &\le & \sum\limits_{t = K+1}^{T}\sum\limits_{r=0}^{\log(t-1)} \mathbb{P} \left\{ \left| \widehat{\mu}_{j, 2^r} - \mu_j \right| > \sqrt{\frac{3\ln(t)}{2^r}}\right\}  \\
&\le & \sum\limits_{t = K+1}^{T}\sum\limits_{r=0}^{\log(t-1)} 2e^{-2 \cdot 2^r \cdot \frac{3\ln(t)}{2^r}} \\
   & \le & O(1) \quad,
    \end{array}
\end{equation}
which concludes the proof.\end{proof}
\begin{lemma}
For any arm $j \in \mathcal{A}$, we have $\sum\limits_{t = K+1}^{T} \mathbb{E} \left[ \bm{1} \left\{ \overline{G_j(t-1)} \right\}  \right] \le O(1) \quad.$
\label{low prob event 2}
\end{lemma}
\begin{proof}(of Lemma~\ref{low prob event 2})
We prove this lemma by using the measure concentration of Laplace random variables (shown in Fact~\ref{fact 1}).  We have
\begin{equation}
    \begin{array}{ll}
    & \sum\limits_{t = K+1}^{T} \mathbb{E} \left[ \bm{1} \left\{ \overline{G_j(t-1)} \right\}  \right] \\
    =& \sum\limits_{t =K+ 1}^{T} \mathbb{P} \left\{ \left| \widehat{\mu}_{j, O_j(t-1)} - \widetilde{\mu}_{j, O_j(t-1)} \right| > \frac{3\ln(t)}{\epsilon \cdot O_j(t-1)}\right\}  \\
    = & \sum\limits_{t =K+ 1}^{T} \mathbb{E} \left[\mathbb{P} \left\{ \left| \widehat{\mu}_{j, O_j(t-1)} - \widetilde{\mu}_{j, O_j(t-1)} \right| > \frac{3\ln(t)}{\epsilon \cdot O_j(t-1)} \mid O_j{t-1}, \widehat{\mu}_{j, O_j(t-1)}\right\}\right] \\
    \le & \sum\limits_{t = K+1}^{T}e^{-3\ln(t)} \\
    \le & O(1) \quad,
    \end{array}
\end{equation}
which concludes the proof.\end{proof}

\section{Proofs for RNM-FTNL}
\label{app:ftnl} 
The full proof for Theorem~\ref{full information DP theorem} is shown in Section~\ref{app:full-info-dp}, the full proof for Theorem~\ref{thm:ftnl-regret} is shown in Section~\ref{app:full-info-upper-bound}, and the full proof for Theorem~\ref{thm:dp-dtol-lower-bound} is shown in Section~\ref{app:lower-bound}.

\subsection{Proofs of Theorem~\ref{full information DP theorem}}
\label{app:full-info-dp}

\begin{proof}(of Theorem~\ref{full information DP theorem})
Let ${X}_{1 : T}$ be the original reward vector sequence and ${X}'_{1 : T}$ be an arbitrary neighbouring reward vector sequence of ${X}_{1 : T}$ such that they differ from each other in at most an arbitrary reward vector. Let us say they differ in the reward vector in round $t$. Let $D_{1:T}$ be the sequence of played actions from    round $1$ to round $T$ when working over   ${X}_{1 : T}$. Similarly, let $D_{1:T}'$ be the sequence of played actions  when working over   ${X}'_{1 : T}$.

For an arbitrary point $\sigma_{1 : T} \in [K]^T$, we claim the probability mass functions of $D_{1:T}$ and $D'_{1:T}$ satisfy
 \begin{equation}
      P \left(  D_{1:T}= \sigma_{1 : T} \mid X_{1:T} \right) \le e^{\epsilon} \cdot P  \left( D_{1:T}'= \sigma_{1 : T} \mid X'_{1:T} \right) \quad.
      \label{dp full 1111}
 \end{equation}
  Let $\tau^{(r)}$ be the last round of epoch $r$, i.e., at the end of round $\tau^{(r)}$, a new action will be recommended. Also,  let  $r_0$ be the epoch including round $t$. Hence, we have $t \in \left\{ \tau^{(r_0 - 1) } + 1,\tau^{(r_0 - 1) } + 2, \dotsc , \tau^{(r_0) }    \right\}$. We set $\tau^{(-1)} = 0$. Note that for a fixed $t$, $r_0$ is also fixed.
    The LHS of (\ref{dp full 1111}) can be rewritten as
 \begin{equation}
 \begin{array}{lll}
 P \left( D_{1:T} = \sigma_{1 : T} \mid X_{1:T}  \right) &= &
 \mathop{\prod}\limits_{0 \le r \le r_0} P \left( D_{\tau^{(r-1)}+1:\tau^{(r)}} = \sigma_{\tau^{(r-1)}+1:\tau^{(r)}} \mid X_{1:T}   \right) \\
 
 &&\\
 &&\underbrace{P \left(D_{\tau^{(r_0) }+1:\tau^{(r_0 + 1)}}= \sigma_{\tau^{(r_0) }+1 : \tau^{(r_0 + 1)}}  \mid X_{1:T}\right)}_{\alpha}\\ 

 &&\mathop{\prod}\limits_{r \ge r_0 + 2} P\left( D_{\tau^{(r-1)}+1:\tau^{(r)}}= \sigma_{\tau^{(r-1)}+1:\tau^{(r)}}  \mid X_{1:T}\right)\quad.
 \end{array}
 \label{balabala 1}
 \end{equation}
   Similarly, the RHS of (\ref{dp full 1111}) can be 
  rewritten as
 \begin{equation}
     \begin{array}{lll}
P \left( D'_{1:T} = \sigma_{1 : T} \mid X'_{1:T}    \right) &= &
 \mathop{\prod}\limits_{0 \le r \le r_0} P \left( D'_{\tau^{(r-1)}+1:\tau^{(r)}} = \sigma_{\tau^{(r-1)}+1:\tau^{(r)}} \mid X'_{1:T}  \right) \\
 &&\\
 &&\underbrace{P  \left( D'_{\tau^{(r_0) }+1:\tau^{(r_0 + 1)}}= \sigma_{\tau^{(r_0) }+1 : \tau^{(r_0 + 1)}}  \mid X'_{1:T} \right)}_{\alpha'}\\
 &&\\
&&\mathop{\prod}\limits_{r \ge r_0 + 2} P \left( D'_{\tau^{(r-1)}+1:\tau^{(r)}}= \sigma_{\tau^{(r-1)}+1:\tau^{(r)}}  \mid X'_{1:T}\right)\quad.
     \end{array}
     \label{balabala 2}
 \end{equation}
      The idea behind the rewriting is  the decisions made in epoch $r$ only depends on the reward vectors and injected noise in epoch $r-1$. The decisions made in epoch $r$ and epoch $r-1$ are independent.
           Since ${X}_{1:T}$ and ${X}'_{1:T}$ can only differ in round $t$, the distribution for the played actions from round $1$ to round $\tau^{(r_0) }  $ (including round $\tau^{(r_0) } $) 
      {and from round $\tau^{(r_0 + 1)}+1$ to the end stays the same under either ${X}_{1:T}$ or ${X}'_{1:T}$}.
   Therefore,  in (\ref{balabala 1}) and (\ref{balabala 2}), only $\alpha$ and $\alpha'$ can be different. 
            Now, we analyze the relationship of $\alpha$ and $\alpha'$.
   Recall $J^{(r_0)}$ indicates the action recommended  at the end of epoch $r_0$ and it will be played in all rounds $t \in \left\{ \tau^{(r_0 ) } + 1, \dotsc , \tau^{(r_0+1) }    \right\}$. 
We have two cases. If $\sigma_t$ for all $t \in \left\{ \tau^{(r_0) } + 1, \dotsc , \tau^{(r_0+1) }    \right\}$ are not the same, we have $\alpha = \alpha' = 0$.  For the case where $\sigma_t$ for all $t \in \left\{ \tau^{(r_0 ) } + 1, \dotsc , \tau^{(r_0+1) }    \right\}$ are the same, we have  
\begin{equation}
    \begin{array}{lll}
         \alpha & = & P \left(D_{\tau^{(r_0) }+1:\tau^{(r_0 + 1)}}= \sigma_{\tau^{(r_0) }+1 : \tau^{(r_0 + 1)}}  \mid X_{1:T}\right)\\
         &= & P \left(J^{(r_0)} = \sigma_{\tau^{(r_0)} + 1}  \mid X_{1:T}\right)\\
         & \le & e^{\epsilon} \cdot  P \left(J^{(r_0)} = \sigma_{\tau^{(r_0)} + 1}  \mid X'_{1:T}\right)\\
         & = & e^{\epsilon} \cdot \alpha'\quad,
          \end{array}
\end{equation}
which concludes the proof. Note that the only inequality uses the result shown in Claim~3.9 in \cite{dwork2014algorithmic} stating that Report Noisy Max  is $\epsilon$-differentially private.
\end{proof}

\subsection{Proofs for Theorem~\ref{thm:ftnl-regret}}
\label{app:full-info-upper-bound}

For our analysis, we first re-express RNM-FTNL using a convenient aggregated reward (or gain) notation. 
Let $G^{(r)}_j$ denote the aggregated reward of action $j$ from the $2^r$ rounds in epoch $r$, with the convention that $G^{(-1)}_j = 0$ for all $j \in [K]$. 
In addition, let $\tilde{G}_j^{(r)}$ be the corresponding differentially private aggregated reward, defined as the sum of $G_j^{(r)}$ and an independent, zero-mean Laplace noise random variable with the scale parameter $b = 1/\epsilon$. Using this notation, at the start of any epoch $r \geq 1$, RNM-FTNL plays the action $j$ for which $\tilde{G}_j^{(r-1)}$ is maximized. Also, recall from Algorithm~\ref{Full Information Algorithm} that RNM-FTNL plays action 1 in epoch 0.

The proof of Theorem~\ref{thm:ftnl-regret} makes use of the following concentration result stated in Lemma~\ref{lemma:ftnl-concentration} below for RNM-FTNL.

\begin{lemma} \label{lemma:ftnl-concentration}
For all $r \geq 0$, $\lev \geq 1$, and $j \in \J_\lev$, we have
\begin{equation}
    \begin{array}{lll}
         \mathbb{P} \left\{ \tilde{G}_j^{(r)} - \tilde{G}_1^{(r)} \geq 0 \right\} & 
\leq & 2 e^{ -2^r \cdot \Delta_{(\lev)} \cdot  \min\{\Delta_{(\lev)}, \epsilon \} / 8 }\quad. 
    \end{array}
\end{equation}
\end{lemma}

\begin{proof}{(of Lemma~\ref{lemma:ftnl-concentration})}
We begin with the observation that
$\tilde{G}_j^{(r)} - \tilde{G}_1^{(r)} = G_j^{(r)} - G_1^{(r)} + \xi_{j,r}$,
where $\xi_{j,r}$ is the sum of two independent Laplace noise random variables each with their scale parameter set to $1/\epsilon$.
Clearly, we have
\begin{equation}
    \begin{array}{lll}
\mathbb{P} \left\{ \tilde{G}_j^{(r)} - \tilde{G}_1^{(r)} \geq 0 \right\} 
&= &\mathbb{P} \left\{ \tilde{G}_j^{(r)} - \tilde{G}_1^{(r)} + 2^r \Delta_j \geq 2^r \Delta_j \right\}  \\
&\leq & \mathbb{P} \left\{ G_j^{(r)} - G_1^{(r)} + 2^r \Delta_j \geq 2^r \Delta_j / 2 \right\} 
          + \mathbb{P} \left\{ \xi_{j,r} \geq 2^r \Delta_j / 2 \right\}. 
    \end{array}
\label{eqn:conc-proof-two-things}
\end{equation}
Next, we observe that $\Delta_j \geq \Delta_{(\lev)} = 2^{\lev-1} \Deltamin$ for all $j \in \J_\lev$. 
From Hoeffding's inequality with $2^r$ i.i.d.~reward differences (each of range at most 2), the first term of \eqref{eqn:conc-proof-two-things} is at most 
\begin{align*}
\mathbb{P} \left\{ G_j^{(r)} - G_1^{(r)} + 2^r \Delta_j \geq 2^r \Delta_j / 2 \right\} \le e^{-2^r \Delta_j^2 / 8} \leq e^{-2^r \Delta_{(\lev)}^2 / 8} \quad.
\end{align*}
For the second term of \eqref{eqn:conc-proof-two-things}, we just naively use a union bound; letting $Z$ be a Laplace random variable with scale parameter set to $1/\epsilon$, we have
\begin{align*}
\mathbb{P} \left\{ \xi_{j,r} \geq 2^r \Delta_j / 2 \right\} 
\leq 2 \mathbb{P} \left\{ Z \geq 2^{r-2} \Delta_j \right\} 
\leq e^{-\epsilon \cdot 2^{r-2} \Delta_j} 
\leq e^{-\epsilon \cdot  2^{r-2} \Delta_{(\lev)}} \quad,
\end{align*}
where we can drop the factor of 2 in the penultimate step because we only need a bound on the upper tail of $Z$.

Applying the above gives
\begin{align*}
\mathbb{P} \left\{ \tilde{G}_j^{(r)} - \tilde{G}_1^{(r)} \geq 0 \right\} 
&\leq e^{-2^r \Delta_{(\lev)}^2 / 8} + e^{-2^r \Delta_{(\lev)} \cdot \epsilon / 4} \leq 2 e^{-2^r \Delta_{(\lev)} \min\{\Delta_{(\lev)}, \epsilon \} / 8} \quad,
\end{align*}
which concludes the proof.\end{proof}

\begin{proof}(of Theorem~\ref{thm:ftnl-regret})
Observe that all plays of actions $j$ for which $\Delta_j \leq \frac{1}{T}$ can collectively contribute at most 1 to the regret, and hence it is safe to assume (ignoring actions as necessary) that $\Deltamin > \frac{1}{T}$. In other words, without loss of generality, we assume all the mean reward gaps are lower bounded by $\frac{1}{T}$.
Since the regret is non-decreasing in $T$, it is also without loss of generality that we increase $T$ if necessary so that $T = 2^{R+1} - 1$ for some nonnegative integer $R$. For any $r \geq 0$, let $\tau(r) := 2^r$ be the first round of epoch $r$. 
Then, the regret of RNM-FTNL can be written as 
\begin{align}
\sum_{t=1}^T \sum\limits_{j=2}^K \mathbb{P} \left\{J_t = j \right\}  \Delta_j 
&= \sum_{r=0}^R 2^r \sum_{j=2}^K \mathbb{P} \left\{J_{\tau(r)} = j \right\}  \Delta_j \nonumber \\ 
&= \sum_{\lev \geq 1: |\mathcal{J}_{\lev}| > 0} \underbrace{\sum\limits_{r=0}^R 2^r \sum\limits_{j \in \J_\lev} \mathbb{P} \left\{J_{\tau(r)} = j \right\}  \Delta_j}_{\lambda_{(\lev)}}. \label{eqn:bread}
\end{align}
We will bound each $\lambda_{(\lev)}$ term separately. 
We have
\begin{align*}
\lambda_{(\lev)} 
= \sum_{r=0}^R 2^r \sum_{j \in \J_\lev} \mathbb{P} \left\{ J_{\tau(r)} = j \right\}  \Delta_j 
\leq 2^\lev \Deltamin  \cdot \underbrace{\sum_{r=0}^R 2^r  \mathbb{P} \left\{ J_{\tau(r)} \in \J_\lev \right\}}_{\beta_{(\lev)}}.
\end{align*}
Let $r_0$ be a positive integer that will be tuned later. Then, we have
\begin{align}
\beta_{(\lev)} &= \sum_{r=0}^R 2^r \mathbb{P} \left\{ J_{\tau(r)} \in \J_\lev \right\} \nonumber \\
&\leq 1 + \sum_{r=1}^R 2^r \mathbb{P} \left\{ \exists j \in \J_\lev \colon \tilde{G}^{(r-1)}_j - \tilde{G}^{(r-1)}_1 \geq 0 \right\} \nonumber \\
&= 1 + \sum_{r=1}^{r_0} 2^r \mathbb{P} \left\{ \exists j \in \J_\lev \colon \tilde{G}^{(r-1)}_j - \tilde{G}^{(r-1)}_1 \geq 0 \right\}
      + \sum_{r=r_0 + 1}^R 2^r \mathbb{P} \left\{ \exists j \in \J_\lev \colon \tilde{G}^{(r-1)}_j - \tilde{G}_1^{(r-1)} \geq 0 \right\} \nonumber \\
&< 2^{r_0 + 1} + \sum_{r=r_0 + 1}^\infty 2^r \mathbb{P} \left\{ \exists j \in \J_\lev \colon \tilde{G}_j^{(r-1)} - \tilde{G}_1^{(r-1)} \geq 0 \right\} \nonumber \\
&\leq 2^{r_0 + 1} + 2 |\J_\lev| \max_{j \in \J_\lev} \sum_{r=r_0}^\infty 2^r \mathbb{P} \left\{\tilde{G}_j^{(r)} - \tilde{G}_1^{(r)} \geq 0 \right\} \nonumber \\
&\leq 2^{r_0 + 1} + 4 |\J_\lev| \max_{j \in \J_\lev} 
\sum_{r=r_0}^\infty 2^r e^{-2^r \Delta_{(\lev)} \min\{\Delta_{(\lev)}, \epsilon \} / 8}, \label{eqn:post-concentration}
\end{align}

where the last inequality bounds the probability using Lemma~\ref{lemma:ftnl-concentration}.

Letting $\alpha_\lev := \Delta_{(\lev)} \min\{\Delta_{(\lev)}, \epsilon \} / 8$, the summation of the last line above can be bounded as
\begin{align*}
\sum_{r=r_0}^\infty 2^r e^{-2^r \alpha_\lev} 
= \sum_{r=r_0}^\infty \sum_{s=1}^{2^r} e^{-2^r \alpha_\lev} 
&= 2 \sum_{r=r_0}^\infty \sum_{s=1}^{2^{r-1}} e^{-2^r \alpha_\lev} \\ 
&\leq 2 \sum_{r=r_0}^\infty \sum_{s=1}^{2^{r-1}} e^{-(2^{r-1} + s) \alpha_\lev} \\
&= 2 \sum_{s=2^{r_0 - 1} + 1}^\infty e^{-s \alpha_\lev} \\
&\leq 2 e^{-2^{r_0 - 1} \alpha_\lev} \sum_{s=0}^\infty e^{-s \alpha_\lev} \\
&= 2 e^{-2^{r_0 - 1} \alpha_\lev} \frac{1}{1 - e^{-\alpha_\lev}} \\
&\leq \frac{4 e^{-2^{r_0 - 1} \alpha_\lev}}{\alpha_\lev}\quad,
\end{align*}
where the last inequality holds since $e^{-x} \leq 1 - \frac{x}{2}$ for $x \in [0, 1]$ combined with the fact that $\Delta_{(\lev)} \leq 2$ for any $\lev$ that we consider (and hence $\alpha_\lev \leq \frac{1}{2}$).

Applying the above inequality in \eqref{eqn:post-concentration}, we have
\begin{align*}
\beta_{(\lev)} = 
\sum_{r=0}^R 2^r \mathbb{P} \left\{ J_{\tau(r)} \in \J_\lev \right\} 
\leq 2^{r_0 + 1} + 128 |\J_\lev| \frac{e^{-2^{r_0 - 1} \min\{\Delta_{(\lev)}^2, \Delta_{(\lev)} \epsilon \} / 8}}{\min\{\Delta_{(\lev)}^2, \Delta_{(\lev)} \epsilon \}} \quad.
\end{align*}
Now, tuning $r_0$ as $r_0 = \left\lceil \log \left( \frac{8 \ln |\J_\lev|}{\min\{\Delta_{(\lev)}^2, \Delta_{(\lev)} \epsilon\}} \right) \right\rceil + 1$, we have
\begin{align*}
\beta_{(\lev)} \leq 
\frac{64 \ln |\J_\lev|}{\min\{\Delta_{(\lev)}^2, \Delta_{(\lev)} \epsilon\}} 
+ \frac{128}{\min\{\Delta_{(\lev)}^2, \Delta_{(\lev)} \epsilon\}} \quad.
\end{align*} 

Recall that $\Delta_{(\lev)} = 2^{\lev-1}\Deltamin$. Then, we have
\begin{align*}
\lambda_{(\lev)} \le 2^{\lev} \Deltamin \cdot \beta_{(\lev)} 
 = 2\Delta_{(\lev)}  \cdot \beta_{(\lev)} = \frac{128 \ln |\J_\lev|}{\min\{\Delta_{(\lev)}, \epsilon\}} 
+ \frac{256}{\min\{\Delta_{(\lev)}, \epsilon\}} 
\leq \frac{256 (1 + \ln |\J_\lev|)}{\min\{\Delta_{(\lev)}, \epsilon\}} \quad.
\end{align*}
Applying the above result in (\ref{eqn:bread}), the regret of RNM-FTNL is at most
\begin{align*}
\sum_{\lev \geq 1 \colon |\J_\lev| > 0} \lambda_{(\lev)} \le 
\sum_{\lev \geq 1 \colon |\J_\lev| > 0} \frac{256 (1 + \ln |\J_\lev|)}{\min\{\Delta_{(\lev)}, \epsilon\}} \quad.
\end{align*}
To get something more interpretable, let us further upper bound the above expression:
\begin{align}
\sum_{\lev \geq 1 \colon |\J_\lev| > 0} \frac{256 (1 + \ln |\J_\lev|)}{\min\{\Delta_{(\lev)}, \epsilon\}}
&\leq \sum_{\lev \geq 1 \colon |\J_\lev| > 0} \frac{256 (1 + \ln |\J_\lev|)}{\Delta_{(\lev)}} 
          + \sum_{\lev \geq 1 \colon |\J_\lev| > 0, \Delta_{(\lev)} > \epsilon} \frac{256 (1 + \ln |\J_\lev|)}{\epsilon} . \label{eqn:regret-almost-there}
\end{align}

The second term already appears in the first bound of the theorem. 
Recalling that $\Delta_{(\lev)} = 2^{\lev-1} \Deltamin$, the first summation on the RHS can be bounded as
\begin{align*}
\sum_{\lev \geq 1 \colon |\J_\lev| > 0} \frac{256 (1 + \ln |\J_\lev|)}{2^{\lev-1} \Deltamin} 
&\leq \sum_{\lev \geq 1} \frac{256 (1 + \ln (K-1))}{2^{\lev-1} \Deltamin} \\
&\leq \sum_{\lev = 0}^\infty \frac{256 (1 + \ln (K-1))}{2^\lev \Deltamin} \\
&= \frac{512 (1 + \ln (K-1))}{\Deltamin} .
\end{align*}
Plugging this bound into \eqref{eqn:regret-almost-there} implies that the regret is at most
\begin{align*}
O \left( 
  \frac{\log (K)}{\Deltamin} 
  + \sum_{\lev \geq 1 \colon |\J_\lev| > 0, \Delta_{(\lev)} > \epsilon} \frac{1 + \log |\J_\lev|}{\epsilon}
\right) ,
\end{align*}
which is the first bound of the theorem.

We now explain how to go from the above bound to the second bound of the theorem. First, for each $\lev$ we use the upper bound $|\J_\lev| \leq K$. Next, we consider the minimum and maximum values of $\Delta_{(\lev)}$ in the summation term. The minimum value is clearly lower bounded by $\max\{\Deltamin, \epsilon\}$; moreover, as argued in the beginning of the proof, we need only consider actions for which $\Delta_j > \frac{1}{T}$ and so the minimum value is also lower bounded by $\frac{1}{T}$. The maximum value is clearly of order at most $\Delta_{\max}$. Hence, the the range of $\Delta_{(\lev)}$ in the summation is of order at most $\frac{\Delta_{\max}}{\max\left\{ \Deltamin, \epsilon, \frac{1}{T} \right\}}$; taking the logarithm bounds the number of terms, implying the result.
\end{proof}

\subsection{Proofs for Theorem~\ref{thm:dp-dtol-lower-bound}}
\label{app:lower-bound}

\begin{proof}{(of Theorem~\ref{thm:dp-dtol-lower-bound})}
We now present the formal proof, using the shorthand $\Delta := \Deltamin$ throughout. First, recall that the lower bound 
$\Omega \left( \frac{\log K}{\Delta} \right)$
comes from Proposition 4 of \cite{mourtada2019optimality}. The remainder of the proof details each step of the proof sketch to establish an $\Omega \left( \frac{\log K}{\epsilon} \right)$ on the regret.

\subsubsection*{Step 1: Constructing a packing} 
In an $\alpha$-packing with respect to some metric, all pairs of elements are at distance strictly greater than $\alpha$. We start with a $\frac{\Delta}{2}$-packing $\mathcal{P} := \{P_1, P_2, \ldots, P_K\}$ (with respect to total variation) of $K$ probability distributions; that is, for each distinct pair $i, j \in [K]$,
\begin{align*}
\sup_{A \subseteq \{0,1\}^K} \left| P_i(A) - P_j(A) \right| > \frac{\Delta}{2} \quad.
\end{align*}
Each distribution in $\mathcal{P}$ is a $K$-fold product distribution where the $i\nth$ component distribution is the Bernoulli reward distribution for action $i$. Distribution $P_j$ is defined by setting the mean reward for all but the $j\nth$ component to be $1/2$ and the mean reward for the $j\nth$ component to be $1/2 + \Delta$. 
Therefore, under $P_j$, action $j$ is uniquely optimal. Note that $\mathcal{P}$ satisfies the $\frac{\Delta}{2}$-packing property since, as is easy to work out, the total variance distance for any pair of distributions in the packing is equal to $\Delta$.\footnote{For the first and second distributions, the witnessing (maximizing) set is $A = \left\{ (1, 0) \times \{0, 1\}^{K-2} \right\}$.}

\subsubsection*{Step 2: Sample Complexity Lower Bound for $\epsilon$-DP Hypothesis Testing} 

This step involves lower bounding the sample complexity of $\epsilon$-DP multiple hypothesis testing when requiring some constant success probability. Because we focus on hypothesis testing, the (randomized) estimators we consider will be \emph{selectors}: a selector is an estimator whose output is restricted to the set
$\mathcal{P}$. 
The main work of this step  
is essentially an application of Corollary~4 of \cite{acharya2021differentially} (``ASZ'' hereafter), itself building on their Theorem~2. Before continuing, we take a brief detour related to Theorem~2 of ASZ.

Because we only consider selectors, in the proof of Theorem~2 of ASZ we can take $\mathcal{P} = \mathcal{V}$ (both symbols are from their notation), which means that the factor of $\frac{1}{2}$ introduced at the beginning of that proof is unnecessary. On the other hand, there is a small error near the end of their (very nice) proof. For clarity, we mention that in the below, ASZ's $M$ is our $K$, and their $D$ is nonnegative (see their Theorem 2 for how $D$ is defined). Specifically, their inequality
\begin{align*}
\frac{0.9 M (M - 1)}{M - 1 + e^{10 \epsilon D}} 
\geq 0.8 M \min \left\{ 1, \frac{M}{e^{10 \epsilon D}} \right\}
\end{align*}
is not always true; indeed, taking $M = 2$ and $\epsilon = 0$ gives $0.9 \geq 1.6$ (of course, a similar contradiction occurs for suitably small $\epsilon$). A simple fix is to use the inequality $\frac{x}{x + y} \geq \frac{1}{2} \min \left\{ 1, \frac{x}{y} \right\}$ for positive $x$ and $y$, giving
\begin{align*}
\frac{0.9 M (M - 1)}{M - 1 + e^{10 \epsilon D}} 
\geq 0.45 M \min \left\{ 1, \frac{M - 1}{e^{10 \epsilon D}} \right\} \quad.
\end{align*}
The above inequality combined with the aforementioned allowable removal of a factor of $\frac{1}{2}$ allows the differentially private part of the lower bound in ASZ's Theorem 2 (in our setting of selectors with the zero-one loss) to be written (in their notation) as
\begin{align*}
R(\mathcal{P}, \ell, \epsilon) \geq 0.45 \min \left\{ 1, \frac{M - 1}{e^{10 \epsilon D}} \right\} \quad;
\end{align*}
here, the quantity $R(\mathcal{P}, \ell, \epsilon)$ is the minimax probability of error.

With the above modification to Theorem~2 of ASZ, the proof of Corollary 4 of ASZ implies that
\begin{align}
\inf_{\widehat{P}}
\max_{P_j \in \mathcal{P}} 
{\mathbb{P}}_{X_{1:n} \sim P_j^n} \left\{ \widehat{P}(X_{1:n}) \neq P_j \right\}
\geq 0.45 \min \left\{ 1, \frac{K - 1}{e^{10 \epsilon n \Delta}} \right\} \quad, \label{eqn:dp-fano-corollary-mod}
\end{align}
where the infimum is over all $\epsilon$-differentially private selectors 
$\widehat{P}$.
 
Consequently, for $n \leq \frac{\log K}{10 \epsilon \Delta}$, the minimax probability of error is at least $0.45 \cdot \frac{K-1}{K} \geq 0.2$.

\subsubsection*{Step 3: Batch-to-Online Conversion} 

Next, we implement a ``batch-to-online'' conversion of the sample complexity lower bound for multiple hypothesis testing. 
Let $\mathcal{S}_K$ be the simplex over $[K]$. 
To see how the conversion works, consider the following method for selecting a hypothesis given a batch of $m$ data points:
\begin{enumerate}
\item Run an online learning algorithm on the batch of data, resulting in a sequence $S \in (\mathcal{S}_K)^m$ of $m$ elements of the simplex.
\item Select a hypothesis whose cumulative weight over the $m$ elements is the maximum, breaking ties arbitrarily.
\end{enumerate}
Clearly, if any hypothesis has cumulative weight in $S$ of at least $0.51 \cdot m$, then this hypothesis will be selected. Therefore, the correct hypothesis having at least this much weight in $S$ suffices for correctness of the selection algorithm.
But from \eqref{eqn:dp-fano-corollary-mod}, any hypothesis selection procedure (including the one described above) requires $m > \frac{\log K}{10 \epsilon \Delta}$ samples in order to have a failure probability less than $0.2$. 
Consequently, with probability at least $0.2$, in the first $m^* := \frac{\log K}{10 \epsilon \Delta}$ rounds the online learning algorithm must have given less than $0.51 \cdot m^*$ cumulative weight to the optimal action (optimal hypothesis). Therefore, with the same probability, in the first $m^*$ rounds the online learning algorithm must have given cumulative weight at least $0.49 \cdot m^*$ to suboptimal actions (suboptimal hypotheses). 
 It follows that in the first $m^*$ rounds, the online learning algorithm picks up regret of at least $0.2 \cdot 0.49 \cdot m^* \cdot \Delta = \Omega \left( \frac{\log K}{\epsilon} \right)$.
\end{proof}

\end{document}